%% file: main_NDSS.tex
\documentclass[conference]{IEEEtran}

\ifCLASSINFOpdf
\else
\fi
\usepackage{amsmath,amssymb,amsfonts} % Math symbols
\usepackage{enumitem}
\usepackage{graphicx}    
\usepackage{amsthm}
\usepackage[ruled,vlined]{algorithm2e}
\usepackage{cite}
\usepackage{tabularx}
\newcommand{\components}{\mathcal{C}}
\usepackage{textcomp}                 % Text symbols
\usepackage{xcolor}                   % Colors
\usepackage{url}                      % URLs
\usepackage{booktabs}                 % Better tables
\usepackage{tabularx}                 % For tables with adjusted column widths
\usepackage{multirow} 
\usepackage{tcolorbox}
\usepackage{subcaption}               % Subfigures/subtables
\usepackage{balance}                  % Balance columns on last page
\usepackage{tikz}
\usetikzlibrary{matrix, positioning, decorations.pathreplacing, calligraphy}
\usepackage{caption} % Already loaded by IEEEtran? Check needed.
\usepackage{cite}     % Handle citations nicely

\usepackage[hidelinks]{hyperref}       % Hyperlinks (hidelinks removes boxes)
\usepackage[capitalize]{cleveref}    % Smart references (e.g., Fig. 1, Section II) - MUST be loaded after hyperref generally

\newcommand{\MI}{\mathrm{I}} % Mutual Information (Math symbol)
\newcommand{\Prob}{\mathbb{P}} % Probability P (using blackboard bold)
\newcommand{\setW}{\mathcal{W}}

\newcommand{\PFL}{\textsf{PFL}}

\newcommand{\FedAvg}{\textsf{FedAvg}}

\newcommand{\FedProx}{\textsf{FedProx}} % Example Algo
\newcommand{\Scaffold}{\textsf{SCAFFOLD}}
\def\eg{\emph{e.g.}}
\def\ie{\emph{i.e.}}
\def\cf{\emph{cf.}}
\newcommand{\etal}{\textit{et al.}\xspace}

\newtheorem{theorem}{Theorem}
\newtheorem{lemma}[theorem]{Lemma}

\newtheorem{definition}{Definition}

\begin{document}
\title{Choose Wisely and Privately: Proactive Client Selection for Fair and Efficient Federated Learning}

%\author{\IEEEauthorblockN{Anonymous submission}}

\author{
\IEEEauthorblockN{
Adda Akram Bendoukha\IEEEauthorrefmark{1},
Heber Hwang Arcolezi\IEEEauthorrefmark{2},
Nesrine Kaaniche\IEEEauthorrefmark{1},
Aymen Boudguiga\IEEEauthorrefmark{3}
}

\IEEEauthorblockA{\IEEEauthorrefmark{1}
Samovar, Télécom SudParis, Institut Polytechnique de Paris, France\\
Emails: adda-akram.bendoukha@telecom-sudparis.eu,
kaaniche.nesrine@telecom-sudparis.eu
}

\IEEEauthorblockA{\IEEEauthorrefmark{2}
ÉTS Montréal\\
Email: heber.hwang-arcolezi@etsmtl.ca
}

\IEEEauthorblockA{\IEEEauthorrefmark{3}
CEA-LIST\\
Email: aymen.boudguiga@cea.fr
}
}

\maketitle

\begin{abstract}
Federated Learning enables collaborative model training across decentralized data sources without data transfer. Averaging-based FL is limited by the presence of non-IID data, which negatively impacts convergence speed and final model accuracy. Conventional alternatives suffer from significant inefficiency. Clients with noisy or highly heterogeneous data contribute expensive gradient computations that are either discarded or heavily down-weighted before aggregation. These reactive approaches waste computational resources, require more communication rounds and result in unnecessary privacy exposure. 
In this paper, we propose a proactive client selection framework that aims to find an optimal federation of clients whose combined data match utility and fairness requirements before training begins. Our method relies on mutual information computed from differentially private contingency tables to quantify the relevance of cross-feature correlations in the union dataset. We introduce a Potential Federation Loss (PFL) over the set of fixed-size federations, which balances two objectives: maximizing collective data utility while ensuring fair cross-features correlations to prevent group unfairness. Client selection is expressed as an optimal subset search problem over the PFL objective, which we solve using simulated annealing under strong differential privacy guarantees for clients' local statistics. Experimental results on four benchmarks show faster, fairer, and more accurate models trained on optimally found federations, compared to uniform sampling, even when state-of-the-art adaptive aggregation or sampling strategies are employed.
\end{abstract}

\IEEEpeerreviewmaketitle

\input{main_text}

\bibliographystyle{IEEEtran}
\bibliography{biblio} 

\appendices
\input{appendices}

\end{document}

%% file: main_text.tex
\section{Introduction}
\label{sec:intro}
Federated Learning (FL) \cite{mcmahan2023communication} enables collaborative training of machine learning models across distributed data (\eg{}, smartphones, hospitals, organizations) without requiring direct data sharing. While promising for privacy, the standard \FedAvg{} algorithm \cite{mcmahan2023communication}, and its numerous extensions \cite{kairouz2021advances,li2020federated,wang2020tackling} face multiple challenges primarily stemming from disparities in clients' data. Specifically, statistical heterogeneity, where data distributions deviate significantly across clients, is pervasive and can severely degrade the final model performance and convergence speed \cite{kairouz2021advances, li2020federated, zhao2018federated}. Additionally, 
models trained via FL can inherit and even exacerbate discriminative patterns present in the local data \cite{kairouz2021advances, chang2023bias}, leading to performance disparities between demographic groups defined by sensitive attributes like race or sex. 

Existing research addresses these challenges by selectively discarding, down-weighting, or under-sampling certain client updates during the federated learning process. Robust aggregation methods aim to mitigate the impact of non-IID data \cite{li2020federated, pillutla2022federated, blanchard2017machine}. Fairness-aware FL algorithms incorporate fairness constraints or regularization terms into the training objective \cite{ezzeldin2022fair, bendoukha2025fade, zeng2022improving}. 
However, these approaches are largely \emph{reactive}, operating on a fixed pool of participating clients.  
In practice, the client pool may include participants with limited utility, redundant information, or updates derived from low-diversity or biased data.
Consequently, such methods can lead to inefficient training and unnecessary privacy exposure of client information. Specifically, even though local data is never transferred, privacy in conventional FL remains limited. Vulnerabilities like model inversion or membership inference attacks remain largely possible in the federated setting \cite{du2025sokgradientleakagefederated, melis2018exploiting, Rigaki_2023}. 
Engaging a larger than necessary set of clients increases the privacy risk surface, without added value in terms of predictive utility, and can even introduce new convergence challenges due to biased or noisy data.

Privacy in FL is addressed through secure aggregation techniques such as Fully Homomorphic Encryption (FHE) \cite{sav2021poseidon, choffrut2024sable, bendoukha:hal-04782394, bendoukha2025unveiling} and Multi-Party-Computation \cite{bonawitz2016practical}. Setting up these cryptographic methods in FL often leads to a quadratic cost in the number of involved clients. Concrete examples involve collaborative decryption in threshold FHE \cite{sav2021poseidon} and the setup of pairwise masking schemes \cite{bonawitz2016practical}. Such scalability limitations further motivate the need for proactive and data-driven participation policies in FL.

Additionally, existing reactive aggregation and sampling strategies are inherently difficult to reconcile with secure aggregation protocols. These approaches typically rely on inspecting, filtering, or reweighing individual updates based on their observed contribution to the global objective. However, under encryption or masking, individual updates are cryptographically obfuscated, which makes any server-side inspections extremely expensive\footnote{For instance, Choffrut \etal{}~\cite{choffrut2024sable} implemented median-based aggregation rules for FHE encrypted updates. They report aggregation times exceeding one minute for a single coordinate-wise trimmed mean operation with 10 clients.}. 
As a result, while simple averaging remains practical under secure aggregation, more sophisticated reactive strategies become computationally infeasible.
.

These observations motivate \emph{a shift from reactive to proactive} approaches in FL. 
Instead of filtering or down-weighting updates after they are computed and securely transferred, we propose to identify, in advance, a subset of clients whose participation is both efficient and aligned with fairness and utility objectives. 
Concretely, we ask: \emph{can we construct a federation in which every participating client contributes meaningfully, while preserving privacy throughout the selection process?} 
Our goal is to enable such proactive client selection in a privacy-preserving manner, thereby reducing unnecessary computation, limiting privacy exposure, and improving overall training efficiency.

To this end, we introduce a novel approach for proactive client selection that identifies a federation (\ie{}, a subset of clients) whose combined data properties promote both high predictive utility and group fairness (\cf{} Appendix \ref{sec:appendix_fairness_metrics}). 
Our method builds on principles from Exploratory Data Analysis (EDA) for machine learning, where statistical characteristics of data provide insights into its suitability for supervised tasks \cite{guyon2003introduction, zha2023data, torralba2011unbiased}. In particular, we leverage a mutual information-based analysis of data suitability \cite{battiti1994using, hanchuan2005feature, fleuret2004fast}. 
Unlike linear correlation measures, mutual information captures arbitrary statistical dependencies, making it particularly well-suited for

    (i) assessing feature utility for non-linear models, and 

    (ii) flagging potential fairness concerns related to Demographic Parity~\cite{dwork2011fairness}, and Equal Opportunity~\cite{hardt2016equality} principles by measuring the amount of information the sensitive attribute holds about the target or other non-sensitive features, thereby identifying direct bias and potential non-linear proxy effects.

\begin{figure*}[t]
    \centering
    \includegraphics[width=0.8\linewidth]{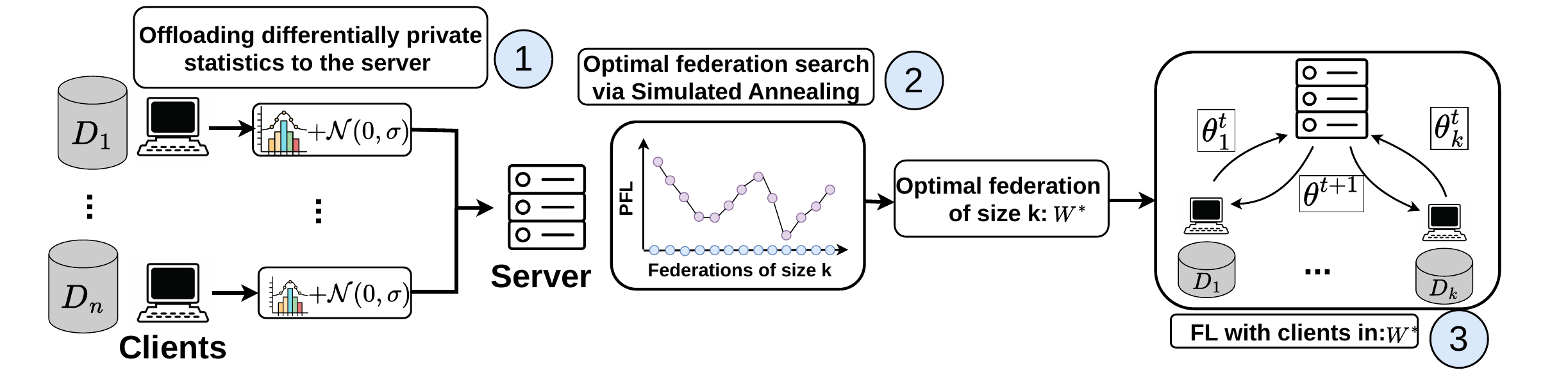}
    \caption{Overview of our approach in three steps. First step involves computing cross-features contingency tables, and adding Gaussian noise to achieve ($\epsilon, \delta$)-DP. The second step involves an optimal federation search using the noisy statistics. Finally, \FedAvg-based FL is performed with the optimal federation $W^{*}$}
    \label{fig:approach_overview}
\end{figure*}

\textbf{Contributions.} The contributions of this work are:

\begin{enumerate}[leftmargin=*]

\item We introduce a \emph{Potential Federation Loss (PFL)}, a novel objective function that evaluates the quality of a candidate federation by jointly modelling predictive utility and fairness risks using mutual information. The objective captures direct bias, proxy bias, feature redundancy, and predictive signal, allowing the identification of federations that balance performance and group fairness.

\item We introduce a \emph{privacy-preserving proactive client selection framework} that constructs an optimal federation of clients before the training phase. Specifically, this search problem is expressed as an optimal subset search over the \PFL{} objective, and solved via Simulated Annealing.

\item We extensively evaluate our approach on American Community Survey datasets~\cite{ding2022retiringadultnewdatasets}, in which U.S. states serve as clients. We compare our method to five reactive aggregation and sampling baselines. The results show that models trained on optimal federations using standard \FedAvg{} outperform reactive strategies applied to randomly selected federations, from fairness and utility aspects.  

\item We provide open-source access to our implementation of the search algorithms and the FL training. We also report all optimally found solutions for full reproducibility at \url{https://github.com/Akram275/FL_private_proactive_selection}.   

\end{enumerate}

\section{Background}
\label{sec:background}

\subsection{Federated learning}
Federated Learning \cite{mcmahan2017fedavg} is a distributed optimization paradigm that enables a central server to train a global model $\theta$ over data distributed among $n$ clients, denoted by the set $\setW = \{1, \dots, n\}$, without any centralized access to raw data $\{D_i\}_{i \in \setW}$.
The goal is to solve a global optimization problem:
\begin{equation*}
    \min_{\theta \in \mathbb{R}^d} F(\theta) = \sum_{i=1}^{n} \left( \frac{|D_i|}{|D|} \right) F_i(\theta)
\end{equation*}
where $|D| = \sum |D_i|$ is the total number of samples, and $F_i(\theta) = \mathbb{E}_{(x,y) \sim D_i} [\ell(\theta; x, y)]$ is the local objective function for client $i$ based on a loss $\ell(\cdot)$.

In the standard \FedAvg{} \cite{mcmahan2017fedavg} algorithm, the server broadcasts the current global model $\theta^{(t)}$ to clients. Each client $i$ performs several Stochastic Gradient Descent (SGD) steps to minimize $F_i(\theta)$ and sends the update $\theta^{(t)}_{i}$ back to the server for aggregation via a coordinate-wise averaging of the updates. 

\subsection{Differential privacy}
Differential Privacy (DP)~\cite{dwork2006differential} provides a mathematical framework for quantifying and achieving privacy for the release of a statistical analysis over a dataset. In particular, it ensures that the output of a randomized algorithm $\mathcal{M}$ on a dataset does not significantly depend on the presence or absence of any single individual's record in the dataset.

\begin{definition}[$(\epsilon, \delta)$-Differential Privacy]
A randomized algorithm $\mathcal{M}$ satisfies $(\epsilon, \delta)$-DP if for all adjacent datasets $D, D'$ (differing by one record) and for all sets of outcomes $S \subseteq \text{Range}(\mathcal{M})$:
\begin{equation*}
    \Prob[\mathcal{M}(D) \in S] \le e^\epsilon \cdot \Prob[\mathcal{M}(D') \in S] + \delta
\end{equation*}
where $\epsilon$ is the privacy budget and $\delta$ is a failure probability.
\end{definition}

A standard method to achieve DP for vector-valued queries is the Gaussian Mechanism \cite{dwork2014algorithmic}, which adds noise calibrated to the $L_2$-sensitivity of the query function $q$ defined as: $\Delta_2(q) = \max_{D, D'} \|q(D) - q(D')\|_2$. Specifically, $\mathcal{M}(D) = q(D) + \mathcal{N}(0, \sigma^2 I)$ satisfies $(\epsilon, \delta)$-DP if $$\sigma \ge \frac{\sqrt{2 \ln(1.25/\delta)} \cdot \Delta_2(q)}{ \epsilon}$$

\textbf{Rényi Differential Privacy (RDP):}
To handle the composition of multiple privacy mechanisms (\eg{}, when releasing multiple local statistics), we employ Rényi Differential Privacy (RDP)~\cite{Mironov_2017}. RDP is a relaxation of pure DP based on the Rényi divergence $\mathcal{D}_\alpha$. A mechanism satisfies $(\alpha, \rho)$-RDP if: 
$$\mathcal{D}_\alpha(\mathcal{M}(D) \| \mathcal{M}(D')) \le \rho$$
RDP provides tighter composition bounds than standard $(\epsilon, \delta)$ analysis \cite{Mironov_2017}, allowing for more accurate noise calibration when releasing multiple high-dimensional statistics.

\section{Related work}
\label{sec:related_work}

Client selection is a central concern in federated learning, primarily due to statistical heterogeneity, limited communication budgets, and demographic fairness concerns. The intuition is that if clients with the best data update the model more often than low-quality or noisy data-owners, convergence is likely to be faster, and leads to an overall better utility model. Early FL frameworks such as \FedAvg{} \cite{mcmahan2017fedavg} rely on uniform sampling of clients at each communication round, which provides scalability without utility and group fairness considerations as soon as strong statistical discrepancies are observed across clients~\cite{mcmahan2017fedavg, kairouz2021advances}. Subsequent works \cite{cho2020bandit, qi2023fedsampling, zhao2018noniid, cho2020client} introduced adaptive client sampling mechanisms that prioritize clients based on the observed gradient utility, local loss \cite{ribero2020communication}, or local data volume \cite{qi2023fedsampling} in order to accelerate convergence. The idea is to progressively bias the sampling probabilities from uniform to more targeted sampling as training progresses.  

Beyond sampling, other strategies rely on constrained local training or adaptive aggregation mechanisms to ensure convergence under heterogeneous conditions. For instance, SCAFFOLD~\cite{karimireddy2021scaffold} introduces control variates to correct for client drift induced by skewed or noisy local data, allowing the model to converge more reliably. FedProx~\cite{li2020federated} constrains local training by introducing a proximal regularization term that penalizes large deviations between local updates and the global model.

Regarding fairness-aware FL, Ezzeldin \etal{}~\cite{ezzeldin2022fair} introduced FairFed, which combines local fairness pre-processing methods with a fairness-aware aggregation rule that penalizes updates that deviate from a global fairness reference. Similarly, Zeng \etal{}~\cite{zeng2022improving} propose FedFB, which adapts the centralized FairBatch \cite{roh2021fairbatch} strategy to federated settings by adjusting local sampling weights to balance group representation during training. Bendoukha \etal{}~\cite{bendoukha2025fade} introduced two fairness-aware strategies based on the observation that, at each training round, the global fairness is determined by the fairness of individual updates. The first strategy computes optimal aggregation weights to balance the global model's unfairness across clients, while the second selects an optimal subset of updates such that the resulting aggregated model exhibits reduced unfairness.

Contrary to these works, our approach adopts a proactive perspective. Instead of reactively down-weighting, under-sampling, or constraining client updates during training, we optimize the federation \emph{before} training begins by selecting a subset of clients whose joint data maximizes predictive utility while minimizing unfairness risks. To the best of our knowledge, such proactive, optimization-based federation search remains largely unexplored. Related ideas have been studied primarily in the centralized setting, where they are often framed under the notion of \emph{Data Valuation} \cite{ghorbani2019datashapley, jia2023efficient, kwon2022betashapley}. These approaches aim to identify optimal subsets of training samples for supervised learning under various constraints, such as efficiency \cite{jia2023efficient} and fairness \cite{kwon2022betashapley}. From this perspective, our work can be viewed as a privacy-preserving federated extension of data valuation, where the objective is to identify an optimal subset of participants rather than individual data samples. A few works exist around the notion of Federated Data valuation. In this context, Li \etal{} proposed FedBary \cite{Li_2024_CVPR}. They frame the selection problem in terms of closeness to an ideal target distribution using Wasserstein distance through an interpolating measure protocol. In contrast to our work, the privacy of FedBary is heuristic, and only claimed through the locality of statistics without formal guarantees. Other prior works, in the context of federated data valuation, either require a server-held validation data \cite{ghorbani2019datashapley, WANG2023213, just2023lava}, a pre-training phase \cite{ghorbani2019datashapley, NEURIPS2021_8682cc30, WANG2023213}, and importantly, to the best of our knowledge, no existing work provides formal privacy analysis, or considers group fairness as a secondary objective in a proactive selection framework.

\section{Proactive optimal federation search}
In this section, we present the core methodology for a privacy-preserving search of an optimal federation that jointly optimizes utility and fairness. 

\subsection{Approach overview}

As illustrated in Figure \ref{fig:approach_overview} and formalized in Algorithm \ref{alg:fl_protocol}, our approach consists of three main phases:

\begin{enumerate}
    \item \textbf{Local counts transmission.} Each client computes contingency tables over pairs of features from its local dataset. To ensure privacy, clients perturb these statistics by adding multi-dimensional, zero-centered Gaussian noise sampled independently for each table before transmission.

    \item \textbf{Optimal federation search.} Upon receiving the noisy contingency tables from candidate clients, the server estimates a federation loss, a metric that quantifies the predictive utility and group fairness potential for any subset of clients. This metric defines the objective of a combinatorial optimization problem, which the server solves using simulated annealing to identify a near-optimal federation as presented in Algorithm \ref{alg:main_approach}.

    \item \textbf{Federated training with the selected federation.} Finally, a model is trained over the selected subset of clients using the \FedAvg{} algorithm.
\end{enumerate}

\subsection{Threat model}
Our framework involves a central server and a set of candidate clients $\mathcal{W}$. The server follows an
\emph{honest-but-curious} (semi-honest) threat model ~\cite{bonawitz2016practical, mcmahan2023communication,
kairouz2021advances}. Under this model, the server honestly executes the protocol: aggregating contingency tables, running
the simulated annealing search, and orchestrating FL with the optimal federation, but attempts to infer information
about clients' data from any information it legitimately receives. Clients are assumed to be honest: they correctly compute and release their local contingency tables.

\begin{algorithm}[h]
\caption{FL with Proactive Selection}
\label{alg:fl_protocol}

\KwIn{
    Candidate client set $\setW$, 
    Initial model parameter $\theta^{(\text{init})}$, 
    Target federation size $k$,
    Privacy budget $\epsilon_1$,
    Training rounds $T$
}
\KwOut{
    Final global model $\theta^T$
}

\BlankLine
\textbf{Phase 1: Differentially Private Analytics} \\
\For{each client $i \in \setW$ in parallel}{
    Compute local contingency tables $\mathcal{C}_i$ from $D_i$ \\
    Add Gaussian noise $\mathcal{N}(0, \sigma^2)$ to $\mathcal{C}_i$ satisfying $(\epsilon, \delta)$-DP (Sec. \ref{ssec:dp_release}) \\
    Send noisy tables $\tilde{\mathcal{C}}_i$ to server 
}

\BlankLine
\textbf{Phase 2: Proactive Federation Search} \\
Server finds optimal federation using $ \{\tilde{\mathcal{C}}_i\}_{i \in \setW} $ 
$W^\star \leftarrow \text{Algorithm~\ref{alg:main_approach}}(\setW, k, \{\tilde{\mathcal{C}}_i\}_{i \in \setW})$ \tcp*{Run SA}

\BlankLine
\textbf{Phase 3: Federated Training with $W^\star$} \\

\,\, $\theta^T \leftarrow \mathrm{FedAvg.FL}(W^\star, T)$

\Return $\theta^T$
\end{algorithm}

\begin{algorithm}[h]
\caption{Optimal Federation Search}
\label{alg:main_approach}

\KwIn{
Client set $\setW$, weights $\alpha, \beta, \gamma, \lambda$, federation size $k$, noisy contingency tables $\{\tilde{\components}_i\}_{i \in \setW}$, initial and final temperatures $\tau_0$, $\tau_{\text{min}}$, cooling rate $\eta$
}
\KwOut{
Selected federation $W^\star \in \Omega_k$
}

\BlankLine
\textbf{Initialization:} \\
Randomly sample an initial federation $W^{(\text{init})} \in \Omega_k$ \\
Set $W^\star \leftarrow W^{(\text{init})}$, temperature $\tau \leftarrow \tau_0$ \\
Compute $\PFL(W^{(\text{init})})$ using aggregated MI values 

\BlankLine
\While{$\tau > \tau_{\min}$}{
    Uniformly sample a neighboring federation $W'$ by swapping a client in $W^{(t)}$
    with one in $\Omega_k \setminus W^{(t)}$ 

    \BlankLine
    \textbf{MI Aggregation:} \\
    For all required variable pairs $(X,Y)$, compute aggregated counts
    $N_{W'}(x,y) = \sum_{i \in W'} N_i(x,y)$ and estimate $\MI_{W'}(X,Y)$ 
    and $\PFL(W')$ via Eq.~\eqref{eq:pfl_def} 

    \BlankLine
    \textbf{Acceptance Rule:} \\
    Let $\Delta = \PFL(W') - \PFL(W^{(t)})$ 
    \eIf{$\Delta \le 0$}{
        Accept the move: $W^{(t+1)} \leftarrow W'$ 
    }{
        Accept the move with probability $\exp(-\Delta / \tau)$ 
    }

    \BlankLine
    \If{$\PFL(W^{(t+1)}) < \PFL(W^\star)$}{
        $W^\star \leftarrow W^{(t+1)}$
    }

    Update temperature $\tau \leftarrow \eta \cdot \tau$
}

\Return $W^\star$
\end{algorithm}

\noindent The following formally defines the \PFL{} objective and describes how it can be computed in a federated manner.

\subsection{Potential Federation Loss (PFL)}
\label{ssec:pfl}

\subsubsection{Global MI via local contingency tables} 
\label{ssec:dist_mi_contingency}

A core component of our framework involves the ability to evaluate the statistical dependencies within a union data $D_W = \bigcup_{i \in W} D_i$ corresponding to a potential federation $W$. We quantify these dependencies using MI and compute the set of cross-feature MI values $\MI_W(X, Y)$ for all relevant pairs of features $(X, Y)$ based on the empirical distribution $\Prob_{D_W}$ of $D_W$. All considered features are assumed to be discrete, either naturally or through binning.

The MI of two discrete random variables $X$ and $Y$, with domains $\mathrm{dom}(X)$ and $\mathrm{dom}(Y)$ respectively is:
\begin{equation}
    \label{eq:mi_definition}
    \MI(X, Y) = \sum_{x \in \mathrm{dom}(X)} \sum_{y \in \mathrm{dom}(Y)} \Prob(x, y) \log_2 \frac{\Prob(x, y)}{\Prob(x) \Prob(y)}
\end{equation}
where $\Prob(x, y)$, $\Prob(x)$, and $\Prob(y)$ are the joint and marginal probability mass functions.

To compute $\MI_W(X, Y)$ using the empirical distribution $\Prob_{D_W}$ from the union dataset $D_W = \bigcup_{i\in W} D_i$, we need the joint frequency counts $N_W(x, y)$. That is, the number of samples in $D_W$ where the feature $X$ has value $x$, and feature $Y$ has value $y$. Computing the MI of $(X, Y)$ is as follows:

\begin{enumerate}[leftmargin=*] 
    \item \textbf{Local Computation:} Each client $i \in \setW$ computes its local contingency tables from its dataset $D_i$. This yields a table of dimension $ \mathrm{dom}(X) \times\mathrm{dom}(Y)$ containing the count $N_i(x, y)$ at the cell $(x, y)$. Let $\components_i$ represent the collection of these local tables communicated by client $i$.

    \item \textbf{Server-Side Aggregation:} When evaluating a specific candidate subset $W$, the server sums the corresponding local contingency table counts received from all clients $i \in W$:
        \begin{equation*}
            N_W(x, y) = \sum_{i \in W} N_i(x, y)
            \label{eq:count_aggregation} 
        \end{equation*}
        This provides the joint counts for the pair $(X,Y)$ over the union dataset $D_W$.

    \item \textbf{Joint and Marginals estimation:} From the aggregated joint counts $N_W(x, y)$ for a given pair $(X,Y)$, the server computes the following:
        \begin{itemize}[leftmargin=*]
            \item Total number of samples in $D_W$: $$n_W = \sum_{x_i \in \mathrm{dom}(X)} \sum_{y_i \in \mathrm{dom}(Y)} N_W(x_i, y_j)$$
            
            \item Joint: $\Prob_{D_W}(x, y) = N_W(x, y) / n_W$
            
            \item Marginals: $$\Prob_{D_W}(x) = \frac{\sum_{y \in \mathrm{dom}(Y)} N_W(x, y)}{n_W}$$ and similarly for $\Prob_{D_W}(y)$.
        \end{itemize}
    
    \item These empirically estimated probabilities are then substituted in \Cref{eq:mi_definition} to obtain $\MI_W(X, Y)$.
\end{enumerate}

\noindent By performing Step 3 for all required pairs of variables $(X, Y)$, the server constructs the complete set of empirical MI measures of $D_W$. It provides a global statistical view from the local statistics of clients in $W$.

\subsubsection{Metric description}
We define a loss function $\PFL: \Omega_k \rightarrow \mathbb{R}$, where $\Omega_k= \{W \subset \mathcal{P}(\setW) : |W|=k\}$, that quantifies the (un)desirability of a federation $W$ of size $k$, based on the MI measures of the associated dataset $D_W$. A low \PFL{} indicates a more desirable federation according to EDA principles \cite{guyon2003introduction, peng2005feature, hardt2016equality}. We categorize the scoring criteria into two primary aspects: utility and group fairness.

\begin{enumerate}[leftmargin=*]
    \item \textbf{Utility:} A high-utility federation provides strong predictive signals without wasting computational resources on repetitive or highly information-redundant data.
    \begin{itemize}[leftmargin=*] 
        \item \emph{High Predictive Signal:} The non-sensitive features should be strongly informative of the target variable ($\MI_W(N_k, T) \gg 0$).
        \item \emph{Low Redundancy:} The non-sensitive features should exhibit low mutual information with each other ($\MI_W(N_k, N_j) \approx 0$). Penalizing redundancy ensures that the selected federation provides diverse and complementary information, avoiding federations where most features carry overlapping information in the union data.
    \end{itemize}

    \item \textbf{Fairness:} The federation should be constructed to minimize risks of discrimination, both direct and indirect.
    \begin{itemize}[leftmargin=*]
        \item \emph{Direct bias:} The sensitive attribute $S$ should share minimal information with the target variable ($\MI_W(S, T) \approx 0$), reducing the risk of the model learning direct discriminatory patterns.
        \item \emph{Indirect bias:} The mutual information between the sensitive attribute and non-sensitive attributes should be minimized ($\MI(S, N_k) \approx 0$). This is crucial for reducing \emph{proxy bias}, where the model discriminates via correlated features (\eg{}, zip code acting as a proxy for race).
    \end{itemize}
\end{enumerate}

Accordingly, we frame the objective function to balance these two aspects:

\begin{align}
    \PFL(W) = & \underbrace{\alpha \cdot \MI_W(S, T)}_{\text{Direct bias}} 
    + \underbrace{\beta \sum_{N_k \in \mathcal{N}} \MI_W(S, N_k)}_{\text{Indirect bias}} \nonumber \\
    & + \underbrace{\gamma \sum_{N_k \neq N_j \in \mathcal{N}} \MI_W(N_k, N_j)}_{\text{Utility: Redundancy}} 
     - \underbrace{\lambda \sum_{N_k \in \mathcal{N}} \MI_W(N_k, T)}_{\text{Utility: Learning signal}}
\label{eq:pfl_def}
\end{align}

\noindent where $\alpha, \beta, \gamma,$ and $\lambda$ are non-negative weights that control the trade-offs between fairness-related objectives, redundancy penalization, and predictive utility.

\subsubsection{Finding optimal weights}
We express the search for the optimal weights of the \PFL{} objective as a meta-optimization problem. Specifically, we aim to maximize the average Spearman rank correlation ($\rho$) between the \PFL{} scores and the measures of utility and group fairness, computed across a sample of federations. The process is defined as follows:

\begin{enumerate}[leftmargin=*]
    \item We randomly sample a set $\mathcal{S} = \{W_1, \dots, W_r\}$ of $r$ federations of varying sizes from different ACS datasets \cite{ding2022retiringadultnewdatasets} and compute the combined MI measures required by the \PFL{} expression for each federation.
    \item We perform \FedAvg{} training with each federation $W$ yielding final models $\{\theta_{W}\}_{W \in \mathcal{S}}$, and evaluate their utility and group fairness on a held-out validation set.
    \item We use Differential Evolution \cite{storn1997differential} to solve the following optimization problem: 
    \begin{align*}
        \underset{\alpha, \beta, \gamma, \lambda}{\arg\max} \quad &\frac{1}{f} \sum_{i=1}^{N_f} \rho\left( \big\{ \PFL{}(W) \big\}_{W \in \mathcal{S}}, \, \big\{ \mathrm{F}_i^{\text{fair}}(\theta_W) \big\}_{W \in \mathcal{S}}\right) \\
        - &\frac{1}{u} \sum_{j=1}^{N_u} \rho\left( \big\{ \PFL{}(W) \big\}_{W \in \mathcal{S}}, \, \big\{ \mathrm{F}_j^{\text{utility}}(\theta_W) \big\}_{W \in \mathcal{S}}\right)
    \end{align*}
    where $f$ and $u$ are the number of fairness and utility metrics considered ($ \mathrm{F}_i^{\text{fair}}$ and $\mathrm{F}_i^{\text{utility}} $), respectively.
\end{enumerate}

\noindent The goal is to identify the weights that optimally align the \PFL{} score with the desired model outcomes. We seek a strong positive correlation with fairness disparities, meaning a low \PFL{} correctly reflects low group fairness disparities. Conversely, we seek a strong negative correlation with utility metrics, ensuring a low \PFL{} aligns with high utility scores. We use weighted accuracy and F1 score for utility, and Equal Opportunity and Model Accuracy Differences (EOD/MAD) for group fairness. This process with $r = 50$ federations yielded $\alpha=2.0$, $\beta=0.89$, $\gamma=0.11$, $\lambda=1.33$. The correlations are illustrated in Figure \ref{fig:pfl_correlation_with_metrics}.

\begin{figure}[h!]
    \centering
    \includegraphics[width=0.9\linewidth]{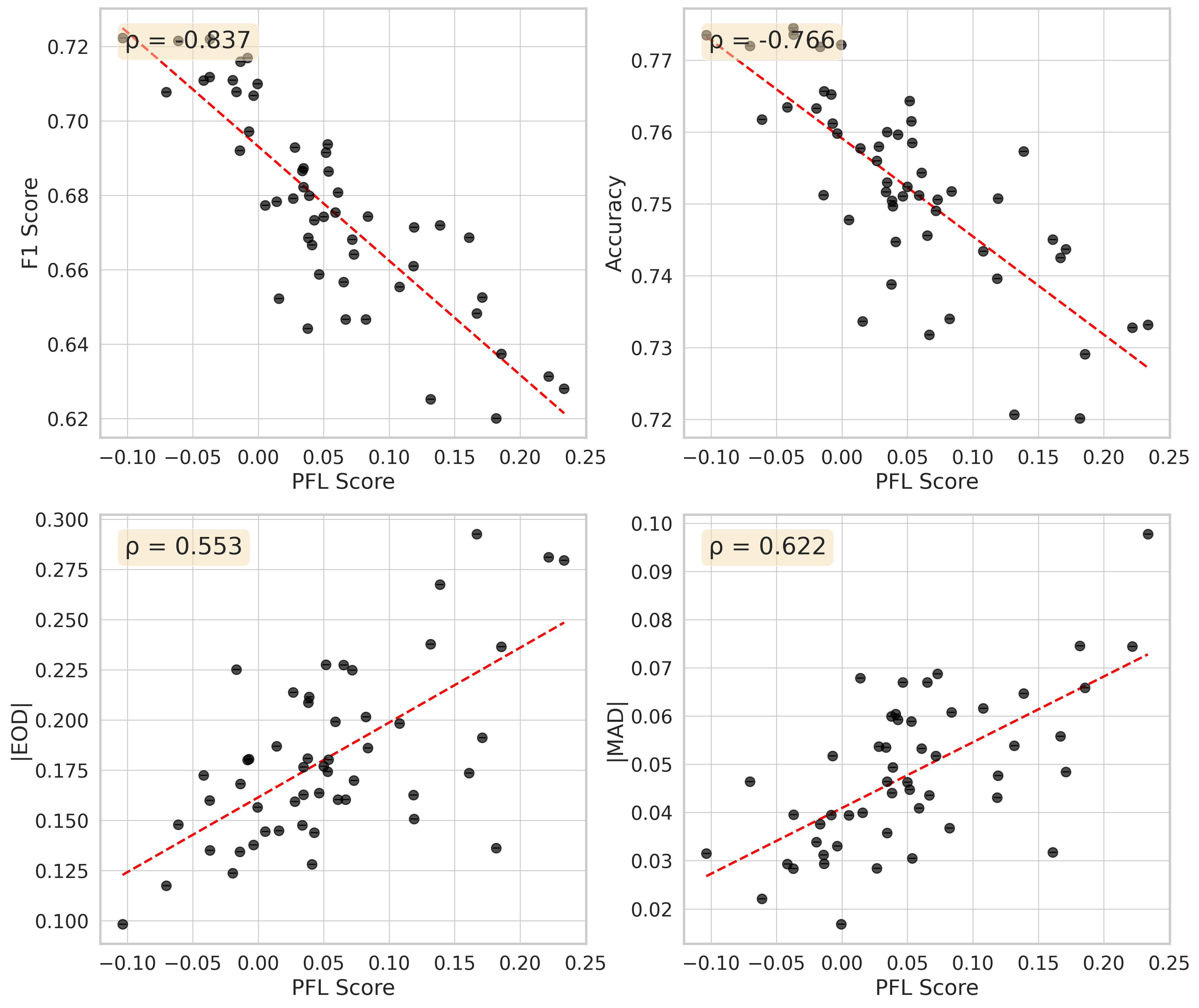}
    \caption{Correlations between the $\PFL{}$ and utility/fairness metrics using weights found through the aforementioned process.}
    \label{fig:pfl_correlation_with_metrics}
\end{figure}

As shown in Figure \ref{fig:pfl_correlation_with_metrics}, the \PFL{} exhibits strong correlations with utility and fairness metrics. Specifically, we observe strong negative correlations for accuracy and the F1 score, yielding Spearman coefficients of $-0.84$ and $-0.77$, respectively. Conversely, the \PFL{} shows strong positive correlations with the fairness disparity measures. This validates that the \PFL{} with optimized weights serves as a highly reliable proxy for a federation's overall quality while remaining efficiently computable from local statistics. However, we note that these values act as domain-specific proxies rather than universal constants. Because MI inherently scales with feature entropy, and the fairness-utility trade-off varies across different empirical data, applying our framework to highly divergent data modalities (\eg{}, high-dimensional medical diagnostics or financial transaction networks) may require a lightweight re-execution of this meta-optimization process to recalibrate the \PFL{} to the new domain's specific statistical properties.

\subsection{Convergence guarantees and interpretation}
\label{ssec:convergence}

Algorithm \ref{alg:main_approach} uses simulated annealing to navigate the combinatorial landscape of client selection. The search space $\Omega_k$ grows combinatorially with $|\setW|$, making exhaustive search infeasible. Furthermore, the objective function $\PFL$ is non-convex and multimodal. Indeed, a client swap might improve the loss locally, but a better combination might exist further in the search space. Under these conditions, greedy heuristics are prone to becoming trapped in poor local minima \cite{battiti1994using, aarts1989simulated}. In this setting, simulated annealing provides the following convergence properties:

\begin{itemize}[leftmargin=*]
    \item \textbf{Asymptotic global convergence:} Modeled as a time-inhomogeneous Markov chain, the algorithm is theoretically guaranteed to converge to the set of global minimizers of $\PFL(W)$ with probability approaching $1$ as $t \to \infty$, provided that the cooling schedule decreases sufficiently slowly. Specifically, this can be achieved using a logarithmic schedule $\tau_t \geq c/\log(t+d)$ \cite{geman1984stochastic, hajek1988cooling}.
    
    \item \textbf{Escaping local minima:} The Metropolis acceptance criterion allows Algorithm \ref{alg:main_approach} to accept worse moves (where $\Delta > 0$) with probability $e^{-\Delta/\tau}$. This stochastic behavior enables the search to escape local minima, particularly during the high-temperature initial phase, which is crucial for avoiding suboptimal federations.
    
    \item \textbf{Ergodicity:} The neighbor generation mechanism defined in Algorithm \ref{alg:main_approach} (swapping a single client) ensures that the underlying Markov chain is irreducible. Any valid federation state $W$ is reachable from any initial state $W^{(\text{init})}$ through a finite sequence of swaps, satisfying the connectivity requirement for convergence.
\end{itemize}

In practice, we employ a geometric cooling schedule ($\tau_{t+1} = \eta \cdot \tau_t$) to ensure termination within a reasonable computational budget. While this relaxes the strict theoretical guarantee, it is a standard approximation in combinatorial optimization that effectively balances \textit{exploration} and \textit{exploitation} with the computational cost\cite{kirkpatrick1983optimization, henderson2003theory}.

\section{Differentially-private contingency tables}
\label{ssec:dp_release}
In this section, we analyse the privacy of our approach with respect to the considered threat model. We begin by providing an overview of the privacy budget throughout different phases, before estimating the required noise.

\subsection{Privacy guarantee structure}

Under the aforementioned threat model, the server's view across the full protocol consists of three components:

\begin{tcolorbox}[
    title=Server view,
    colback=white,
    colframe=black,
    boxrule=0.8pt,
    arc=2pt
]
\begin{enumerate}[leftmargin=*]
    \item \textbf{Phase~1 (Selection Analytics):} Differentially private contingency tables
    $\{\widetilde{\mathcal{C}}_i\}_{i \in \mathcal{W}}$ of all clients.

    \item \textbf{Phase~2 (Federation Search):} The
    \textsf{PFL} values computed during simulated annealing are
    deterministic processes of the noisy tables already held by the server.
    This induces no additional privacy cost due to immunity to post-processing \cite{dwork2006differential}.

    \item \textbf{Phase~3 (FL Training):} The updates
    $\{\theta_i^{(t)}\}_{i \in W^\star}$ received from the clients in $W^\star$ across $T$ rounds.
\end{enumerate}
\end{tcolorbox}

The identity of the selected federation $W^\star$ is a deterministic
function of the first phase's noisy statistics, and therefore does not constitute an
independent release of private information beyond what is already accounted for
in the first phase.

\paragraph{Decomposed privacy budget via sequential composition}

The total privacy guarantee for a client $i \in \mathcal{W}$ decomposes by
sequential composition~\cite{dwork2006differential, Mironov_2017} of the
two phases in which client $i$'s data is accessed, namely Phase-1 and Phase-3.

\textbf{Phase~1} consumes a budget $\epsilon_1$ for the release
of differentially private contingency tables. This is a one-time, pre-training cost incurred
by all candidate clients, regardless of whether they are subsequently
selected.

\textbf{Phase~3} consumes an additional budget of $\epsilon_2$ for clients in $W^\star$ that participate in FL training,
achieved through standard DP-SGD~\cite{Abadi_2016} or secure
aggregation~\cite{bonawitz2016practical}. Clients not selected in
$W^\star$ incur no additional privacy cost after Phase~1.

By sequential composition, the total privacy cost is:
\begin{equation}
    \epsilon_{\mathrm{total}} =
    \begin{cases}
        \epsilon_1 & \text{if } i \notin W^\star \\
        \epsilon_1 + \epsilon_2 & \text{if } i \in
        W^\star
    \end{cases}
    \label{eq:total-budget}
\end{equation}

This decomposition yields a structural privacy advantage over reactive FL
approaches, where all clients are involved throughout every training
round. Non-selected clients are fully protected after Phase~1 with no further
exposure.

\subsection{Privacy budget and noise calibration}
To enable the distributed reconstruction of the MI matrix for any potential federation $W \in \Omega_k$, candidate clients release contingency tables $\mathbf{c}_{X, Y}$ for all relevant attribute pairs $(X, Y)$. To protect local data, we enforce $(\epsilon, \delta)$-differential privacy on each client's release. We use the Gaussian mechanism~\cite{dwork2006differential} together with Rényi Differential Privacy (RDP)~\cite{Mironov_2017} to efficiently allocate the privacy budget over the high-dimensional queries involved in the federation search.

\subsubsection{Sensitivity analysis}
The fundamental unit of our release is the contingency table (or 2D histogram) for a pair of categorical variables. Consider a client $i$ with a local dataset $D_i$ and a pair of variables $(X, Y)$. The contingency table query can be flattened into a vector of counts $\mathbf{c}_{X,Y}: D \rightarrow \mathbb{N}^{|\mathrm{dom}(X)| \times |\mathrm{dom}(Y)|}$, where each entry corresponds to the frequency of a specific value pair $(x, y)$. Since each sample contributes to exactly one cell in the contingency table (corresponding to their specific attribute values), adding or removing a single record changes exactly one coordinate in the vector by $\pm 1$. Consequently, the $L_2$-sensitivity of a single contingency table query is:
\begin{equation*}
    \Delta_2(\mathbf{c}_{X,Y}) = \max_{D, D'} \| \mathbf{c}_{X,Y}(D) - \mathbf{c}_{X,Y}(D') \|_2 = 1
\end{equation*}
Importantly, this holds regardless of the sizes of $\mathrm{dom}(X)$ and $\mathrm{dom}(Y)$.

\subsubsection{Privacy budget estimation via RDP}
To construct the full MI matrix, a client must release tables for all pairwise combinations of $K$ features and the target variable $T$. The total number of queries $M$ corresponds to the number of unique pairs in the set of variables $\{X_1, \dots, X_K, T\}$:
\begin{equation*}
    M = \binom{K+1}{2} = \frac{(K+1)K}{2}
\end{equation*}
We apply the Gaussian Mechanism, adding noise $\mathcal{N}(0, \sigma^2)$ to each cell of every table. Using RDP, which offers tighter composition bounds than standard $(\epsilon, \delta)$ accounting, we analyze the cumulative privacy cost.

The RDP of the Gaussian mechanism with sensitivity $\Delta_2=1$ and noise scale $\sigma$ at order $\alpha$ is given by $\rho(\alpha) = \frac{\alpha}{2\sigma^2}$. Since the noise is added independently to each of the $M$ queries, the total RDP budget accumulates linearly:
\begin{equation}
    \rho_{\text{total}}(\alpha) = \sum_{j=1}^{M} \rho_j(\alpha) = M \cdot \frac{\alpha}{2\sigma^2}
\end{equation}
To satisfy a target $(\epsilon, \delta)$-DP guarantee, we convert the RDP budget back to standard DP parameters using the standard conversion:
\begin{equation}
    \epsilon(\delta) = \min_{\alpha > 1} \left( \frac{M \alpha}{2\sigma^2} + \frac{\ln(1/\delta)}{\alpha - 1} \right)
\end{equation}
This expression is convex in $\alpha$ of the form $f(\alpha) = A\alpha + \frac{B}{\alpha - 1}$, with $A=\frac{M}{2\sigma^2}$ and $B = \ln(1/\delta)$.
By solving $\frac{df}{d\alpha} = A - \frac{B}{(\alpha - 1)^2} = 0$, we obtain the optimal order $\alpha^* = 1 + \sqrt{\frac{B}{A}}$.
Substituting $\alpha^*$ back into the objective yields the analytical relationship~\cite{Mironov_2017}:
\begin{equation}
    \sigma = \frac{\sqrt{2M \ln(1/\delta)} + \sqrt{2M (\ln(1/\delta) + \epsilon)}}{2\epsilon}
    \label{eq:closed_form_sigma}
\end{equation}
This formulation enables us to calibrate the required noise scale $\sigma$ for a fixed $\epsilon$ and number of features, ensuring that the cumulative leakage from releasing the full correlation structure remains within the chosen privacy budget.

\subsubsection{Noise scale calibration}
\label{sssec:noise_calibration}

For our optimal federation search, we impose a strict privacy budget of $\epsilon_1 = 1.0$ for the entire client's release, with $\delta = 10^{-5}$. 

To achieve this target, we must determine the minimum noise scale $\sigma$ required for the Gaussian mechanism. As established in Equation \eqref{eq:closed_form_sigma}, the relationship between $\epsilon_1$, the number of features $K$, and $\sigma$ is non-linear and depends on the minimization over the RDP order $\alpha$.

\textbf{Example Estimation ($K=10$ ACSIncome \cite{ding2022retiringadultnewdatasets}):}
Consider a scenario where the dataset contains $K=10$ features and a label.
\begin{enumerate}
    \item \textbf{Query Count:} The client must release tables for all unique pairs. That is,
     $M = \binom{10+1}{2} = 55$ total contingency tables.
    \item \textbf{Optimization:} We numerically solve for $\sigma$ such that:
    \[ \min_{\alpha > 1} \left( \frac{55 \alpha}{2\sigma^2} + \frac{\ln(10^5)}{\alpha - 1} \right) \le 1.0 \]
    \item \textbf{Result:} Using Equation \eqref{eq:closed_form_sigma} yields a required noise scale of approximately $\sigma \approx 36.5$. This means every cell count in the released tables will have Gaussian noise added with a standard deviation of $36.5$.
\end{enumerate}

\textbf{Impact of Feature Dimensionality:}
As the number of features $K$ increases, the number of required pairwise queries $M$ grows quadratically ($M \propto K^2$). Consequently, to maintain the same strong aggregate privacy level ($\epsilon_1=1.0$), the noise added to each table must increase. Table \ref{tab:sigma_scaling} presents the required noise scale $\sigma$ for varying numbers of features under the fixed constraints $\epsilon_1 = 1.0, \delta = 10^{-5}$.

\begin{table}[h]
\centering
\caption{Required Gaussian Noise Scale ($\sigma$) vs. Number of Features ($K$) for fixed total budget $\epsilon_1 = 1.0, \delta = 10^{-5}$.}
\label{tab:sigma_scaling}
\renewcommand{\arraystretch}{1.2}
\begin{tabular}{ccc}
\hline
\textbf{Features ($K$)} & \textbf{Total Pairs ($M$)} & \textbf{Required $\sigma$}  \\ \hline
5  & 15   & 18.2   \\
10 & 55   & 36.5   \\
20 & 210  & 71.8   \\
30 & 465  & 107.4  \\
50 & 1275 & 179.2  \\ \hline
\end{tabular}
\end{table}

Table \ref{tab:sigma_scaling} shows that RDP enables us to keep the noise growth linear in the number of features as expected from Equation \eqref{eq:closed_form_sigma}.

\subsection{Justification of $\epsilon_1 = 1.0$}

To validate our privacy guarantees empirically, we conduct a membership inference 
audit based on a likelihood ratio test applied to the released noisy contingency 
tables following a similar methodology as Homer \etal{}~\cite{homer2008resolving} attack on aggregate statistics.
The attacker exploits the residual signal left in the noisy marginal counts to 
distinguish members from non-members, using both single-table and joint multi-table 
variants of the attack.
Our main results (Figure~\ref{fig:dp_audit}) show that the attack remains indistinguishable from random guessing for $\epsilon_1 \leq 1.0$, and only becomes non-trivial beyond $\epsilon_1 = 1.0$, confirming that $\epsilon_1 = 1.0$ provides a sufficient privacy guarantee in practice.
Full details of the attack construction and per-task results are provided in 
Appendix~\ref{sec:appendix_dp_audit}.

\section{Noise propagation from counts to PFL}
\label{sec:noise_propagation}

We aim to understand the impact of DP noise on the optimization landscape of Algorithm \ref{alg:main_approach}. We begin by analysing the noise propagation from counts to the \PFL{} objective, before providing our main stability results in Section~\ref{sec:noisy_sa}.

\subsubsection{Noise on aggregated counts}

Let $\mathbf{c}_W \in \mathbb{R}^d$ denote the vector of true joint counts for a
federation $W$, obtained by summing the local contingency tables of its
clients:
$$
\mathbf{c}_W = \sum_{i \in W} \mathbf{c}_i
$$
Each client releases a DP-perturbed table
$$
\tilde{\mathbf{c}}_i = \mathbf{c}_i + \mathbf{z}_i, \qquad \mathbf{z}_i \in \mathbb{R}^d  \text{  and  } \mathbf{z} \sim \mathcal{N}(0, \sigma^2 I)
$$

\noindent with independent Gaussian noise. The aggregated noisy counts therefore satisfy
\begin{equation*}
    \tilde{\mathbf{c}}_W = \sum_{i \in W}  \tilde{ \mathbf{c} }_i
    = \mathbf{c}_W + \underbrace{\sum_{i \in W} \mathbf{z}_i}_{\mathbf{z}_W}
\end{equation*}

By the stability of the Gaussian distribution under summation: $\mathbf{z}_W \sim \mathcal{N}(0, k\sigma^2 I)$.
Hence, the DP mechanism induces an additive centered multivariate Gaussian
perturbation whose variance grows linearly with the federation size $k$.

\subsection{Propagation to mutual information via the Delta method}

Let $f(\mathbf{c}) = \MI(X,Y)$ denote the mapping from a (flattened) joint-count vector $\mathbf{c}$ to the
corresponding mutual information. We aim to characterize the perturbation $\xi_{\MI} = f(\tilde{\mathbf{c}}) - f(\mathbf{c})$.

Assuming a high signal-to-noise regime (aggregated counts large relative to
$\sigma\sqrt{k}$), we consider the unbiased Gaussian perturbation $\tilde{\mathbf{c}}_W$ as a neighborhood of $\mathbf{c}_W$. Thus, a first-order multivariate Taylor expansion around $\mathbf{c}_W$ provides
\begin{equation*}
    f(\tilde{\mathbf{c}}_W) \approx f(\mathbf{c}_W) + \nabla f(\mathbf{c}_W)^\top \mathbf{z}_W
\end{equation*}

Because $\mathbf{z}_W$ is Gaussian and the gradient is deterministic for fixed data, the MI perturbation is itself a scalar Gaussian random variable obtained as a
linear transformation of $\mathbf{z}_W$.

\paragraph{The sufficiency of first-order approximation.}

Assume that $f$ is continuously differentiable with a locally Lipschitz 
gradient, so that there exists $L>0$ such that
$$
\|\nabla f(\mathbf{c}_W + \mathbf{z}_W) - \nabla f(\mathbf{c}_W)\| \le L \|\mathbf{z}_W\|
$$
in a neighborhood of $\mathbf{c}_W$. A standard Taylor bound provides

$$
|R_2(\mathbf{z}_W)| = |f(\mathbf{c}_W + \mathbf{z}_W) - f(\mathbf{c}_W) - \nabla f(\mathbf{c}_W)^\top \mathbf{z}_W| \le \frac{L}{2}\|\mathbf{z}_W\|^2
$$

To justify a first-order approximation, the second-order term  must be 
negligible compared to the linear variation $|\nabla f(\mathbf{c}_W)^\top \mathbf{z}_W|$. 
Since $\mathbf{z}_W$ is Gaussian, the scalar projection $\nabla f(\mathbf{c}_W)^\top \mathbf{z}_W$ 
is itself a centered normal random variable with variance 
$\|\nabla f(\mathbf{c}_W)\|^2 k \sigma^2$. 
Hence, with high probability:

$$
|\nabla f(\mathbf{c}_W)^\top \mathbf{z}_W| = \Theta\bigl(\|\nabla f(\mathbf{c}_W)\|\,\|\mathbf{z}_W\|\bigr)
$$
providing a probabilistic lower bound on the magnitude of the linear term.

Combining this with the Taylor remainder bound yields that, with high probability,
$$
\frac{|R_2(\mathbf{z}_W)|}{|\nabla f(\mathbf{c}_W)^\top \mathbf{z}_W|} = O\left(\frac{L\|\mathbf{z}_W\|}{\|\nabla f(\mathbf{c}_W)\|}\right)
$$
Hence the first-order approximation is valid whenever
$$
\|\mathbf{z}_W\| \ll \frac{\|\nabla f(\mathbf{c}_W)\|}{L},
$$
which defines a neighborhood of local linearity. 

In entropy-type functions satisfy 
$\|\nabla f(\mathbf{c}_W)\| = O(1/n_W)$ and $L = O(1/n_W^2)$ \cite{Fienberg2010differential, awan2024structure}, 
so the linearity radius scales as $O(n_W)$. 
Since differential-privacy noise obeys $\|\mathbf{z}_W\| = O(\sigma\sqrt{k})$, the first-order approximation therefore, holds whenever
$
\sigma \sqrt{k} \ll n_W
$
corresponding to a high signal-to-noise ratio between the true counts and 
the DP perturbation.

Under this condition, mutual information computed on the differentially private aggregated counts is an unbiased estimation of the true mutual information, justifying the Gaussian noise modelling used in the subsequent stability analysis.

\paragraph{Estimating propagated noise variance}
Recall that $\MI(X,Y) = H(X) + H(Y) - H(X,Y)$.
For probabilities $p_k = n_k/n$, entropy can be written explicitly as a
function of the counts:
$$
H = -\sum_k p_k \log p_k
  = -\sum_k \frac{n_k}{n} \log \frac{n_k}{n}.
$$

Differentiating with respect to a specific count $n_k$ requires considering both:
(1) the dependence of $p_k$ on $n_k$, and  
(2) the dependence through the normalization
$n = \sum_j n_j$. Using the quotient rule:
$$
\frac{\partial p_k}{\partial n_k}
= \frac{1}{n} - \frac{n_k}{n^2}
= \frac{1}{n}(1 - p_k)
$$
and for $j \neq k$:
$$
\frac{\partial p_j}{\partial n_k}
= -\frac{n_j}{n^2}
= -\frac{p_j}{n}
$$

Applying the chain rule to
$H = -\sum_j p_j \log p_j$ gives
$$
\frac{\partial H}{\partial n_k}
= -\sum_j
\frac{\partial p_j}{\partial n_k}
(\log p_j + 1)
$$

Substituting the derivatives above and simplifying yields:
$$
\frac{\partial H}{\partial n_k}
= -\frac{1}{n}(\log p_k + H)
$$

Applying this relation to the joint and marginal entropies and accounting for
the dependence of the total count $n = \sum_{x,y} n_{xy}$ on each cell, the
partial derivative of MI with respect to a joint count $n_{xy}$ is
\begin{equation}
\frac{\partial \MI}{\partial n_{xy}}=
-\frac{1}{n}\bigl(\mathrm{PMI}(x,y) + \MI(X,Y)\bigr)
\label{eq:mi_gradient_correct}
\end{equation}
where $\mathrm{PMI}(x,y) = \log \frac{p_{xy}}{p_x p_y}$ denotes the pointwise mutual information (a single event).

\subsection{Resulting variance of the MI perturbation}

Using Eq.~\eqref{eq:mi_gradient_correct} with the covariance of
$\mathbf{z}_W$, the first-order variance of the Gaussian perturbation affecting the
mutual information estimate of a count vector $\mathbf{c}_W$ with total samples $n_W$ is:
\begin{align*}
    \sigma_{\MI}^2
    &= \nabla f( \mathbf{c}_W )^\top (k\sigma^2 I)\, \nabla f(\mathbf{c}_W)
    \\ &= \frac{k\sigma^2}{n_W^2} \sum_{x \in \mathrm{dom}(X)} \sum_{y \in \mathrm{dom}(Y)}
      \left(\mathrm{PMI}(x,y) + \MI\right)^2
\end{align*}

Under finite-alphabet and interior-probability assumptions, the summation
term is $O(1)$, yielding the asymptotic scaling
$$
\sigma_{\MI}^2 = O\!\left(\frac{k\sigma^2}{n_W^2}\right)
$$

Since the \PFL{} is defined as a weighted sum of
multiple noisy mutual-information estimates (\cf{} Equation \ref{eq:pfl_def}), the induced perturbation at the
\PFL{} level also follows a centered Gaussian approximation, and satisfies the same asymptotic scaling
$\sigma_{\PFL}^2 = O\!\left(\frac{k\sigma^2}{n_W^2}\right)$.
Consequently, the noise affecting the optimization objective decreases
quadratically with the total sample size, providing a robust estimation
of the \PFL{} even under high differential-privacy noise. This
result justifies modeling the observed optimization landscape of the simulated annealing instance as an unbiased Gaussian perturbation of the true objective in the following stability analysis.

\subsection{Validation via Monte Carlo simulation}
To empirically validate the Gaussian first-order noise model, we conduct
$1000$ simulations of the following:
\begin{enumerate}[leftmargin=*]
    \item Generate synthetic contingency tables representing aggregated
    federation counts $\mathbf{c}_W$.
    
    \item Compute the MI from the resulting clean counts.
    
    \item Add centered Gaussian noise $\mathbf{z}_W$ with varying standard deviations
    to obtain perturbed counts $\tilde{\mathbf{c}}_W = \mathbf{c}_W + \mathbf{z}_W$.
    
    \item Recompute the MI from the noisy counts
    $\tilde{\mathbf{c}}_W$.
    
    \item Analyze the distribution of the MI perturbation
    $f(\tilde{\mathbf{c}}_W) - f(\mathbf{c}_W)$ as a function of the signal-to-noise ratio (SNR),
    defined as the mean count magnitude divided by the noise standard
    deviation.
\end{enumerate}

\noindent Figure~\ref{fig:noise_at_mi_montecarlo} empirically confirms the previous
analysis.  
In the high SNR regime, where the Gaussian perturbation is small relative to
the total counts, the noisy mutual information provides an accurate and
essentially unbiased approximation of the true MI. This behavior is consistent
with the first-order Delta-method prediction.  
Conversely, in the low–SNR regime, the noisy MI exhibits a noticeable bias and
a substantially increased variance around its expectation, reflecting the
breakdown of the local linear approximation.

\begin{figure}[h]
    \centering
    \includegraphics[width=0.7\linewidth]{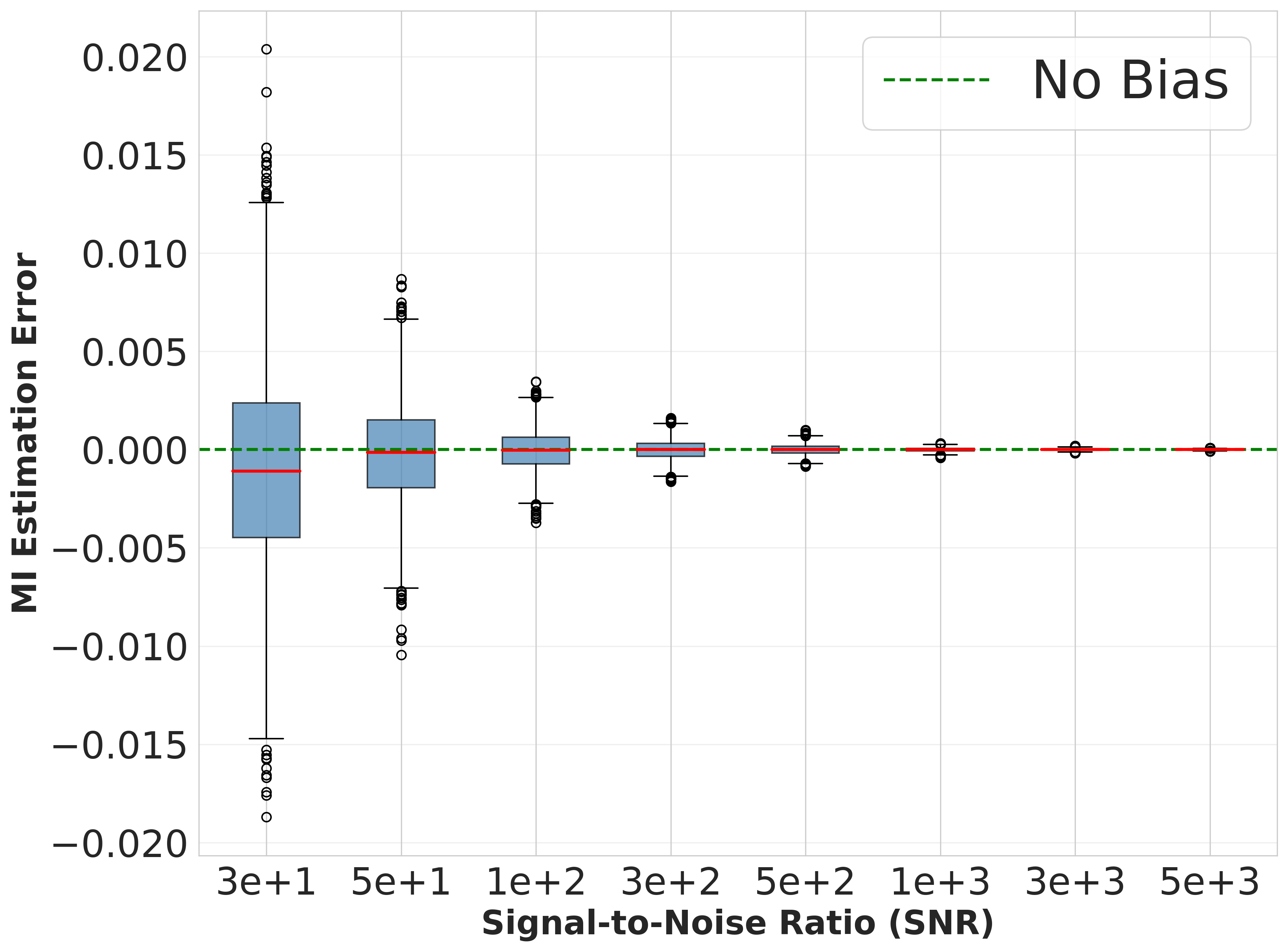}
    \caption{Distribution of MI estimation error under Gaussian perturbations of contingency counts, shown across varying SNR ratios over 1000 simulations per SNR-level.}
    \label{fig:noise_at_mi_montecarlo}
\end{figure}

\section{Simulated annealing over a noisy objective}
\label{sec:noisy_sa}

In our privacy-preserving setting, the server optimizes a noisy estimate of the
\PFL{} objective. As shown in the previous section, using a first-order Taylor expansion in high-SNR regime, the noise at the MI level decomposes into a sum of independent centered Gaussians.
\begin{align*}
    f(\tilde{\mathbf{c}}_W) &\approx f(\mathbf{c}_W) + \nabla f(\mathbf{c}_W)^\top \sum_{i \in W}\mathbf{z}_i 
    \\ &=  f(\mathbf{c}_W) + \sum_{i \in W} \nabla f(\mathbf{c}_W)^\top \mathbf{z}_i
    %\xi_W \;\approx\; \frac{1}{k} \sum_{i \in W} \zeta_i
\end{align*}
The \PFL{} being a weighted sum of multiple noisy MI estimates then,
\begin{equation*}
    Y(W) = \PFL(W) + \xi_W \text{\, with \,} \xi_W = \sum_{i \in W} \xi_i
\end{equation*}
and $\xi_W \sim \mathcal{N}(0,\sigma_{\PFL{}}^2)$ denoting the
effective Gaussian perturbation which decomposes into independent noise source from each client. We will now investigate simulated annealing dynamics on the noisy objective $Y$ with respect to the exact $\PFL{}$. 

\subsection{Variance Reduction via Neighbor Correlation}
Simulated Annealing decisions depend on the noisy objective difference between a
current federation $W$ and a neighboring federation $W'$:
\begin{equation*}
    d = Y(W') - Y(W).
\end{equation*}

Under the additive decomposition above, the covariance structure is
\begin{equation*}
    \mathrm{Cov}(\xi_W,\xi_{W'}) =
    \frac{|W \cap W'|}{k}\,\sigma_{\PFL{}}^2
\end{equation*}
For a single swap move, $|W \cap W'| = k-1$, which yields
\begin{equation*}
    \mathrm{Cov}(\xi_W,\xi_{W'}) =
    \frac{k-1}{k}\,\sigma_{\PFL{}}^2
\end{equation*}

Therefore, although each noisy objective has variance $\sigma_{\PFL{}}^2$, the
variance of the decision variable $d$ is significantly reduced:
\begin{align}
    \mathrm{Var}(d)
    &= \mathrm{Var}(\xi_{W'}) + \mathrm{Var}(\xi_W)
       - 2\,\mathrm{Cov}(\xi_W,\xi_{W'}) \nonumber \\
    &= 2\sigma_{\PFL{}}^2
       - 2\left(\frac{k-1}{k}\right)\sigma_{\PFL{}}^2
       = \frac{2}{k}\sigma_{\PFL{}}^2
    \label{eq:reduced_variance}
\end{align}

Hence, the noise affecting simulated annealing decisions decreases as
$O(1/k)$, implying increased robustness for larger federations.

\subsection{Exponential Stability of Local Decisions}

We define a \emph{misordering event} as the situation where a strictly worse neighbor ($\Delta = \PFL(W') - \PFL(W) > 0$) appears better strictly due to noise, that is, $d \leq 0$ while $\Delta \geq 0$.

\begin{theorem}[Local Stability]
The probability that Simulated Annealing
accepts a strictly worse neighboring federation due solely to noise decays
exponentially with the federation size $k$. For any true objective gap
$\Delta > 0$,
\begin{equation*}
    \Pr(d \le 0 \mid \Delta)
    \;\le\;
    \exp\!\left(
        - \frac{k \Delta^2}{4\sigma_{\PFL{}}^2}
    \right)
\end{equation*}
\end{theorem}

\begin{proof}
From~\eqref{eq:reduced_variance}, we have $d \sim \mathcal{N}\!\left(
        \Delta,\; \frac{2}{k}\sigma_{\PFL{}}^2
    \right)$. Thus,
\begin{equation*}
    \Pr(d \le 0 \mid \Delta)
    = \Pr\!\left(
        Z \le -\frac{\Delta}{\sqrt{2\sigma_{\PFL{}}^2/k}}
    \right)
\end{equation*}
where $Z \sim \mathcal{N}(0,1)$. Applying the standard sub-Gaussian tail bound
$\Pr(Z \le -t) \le e^{-t^2/2}$ yields
\begin{equation*}
    \Pr(d \le 0 \mid \Delta)
    \le
    \exp\!\left(
        - \frac{\Delta^2}{2 \cdot (2\sigma_{\PFL{}}^2/k)}
    \right)
    =
    \exp\!\left(
        - \frac{k \Delta^2}{4\sigma_{\PFL{}}^2}
    \right)
\end{equation*}
\end{proof}

This shows that local comparisons under the noisy objective remain highly
stable. Indeed, the probability of misordering neighboring federations decreases
exponentially both in the federation size $k$ and in the squared true objective gap $\Delta^2$.

\subsection{Global $\mu$-optimality via Gaussian supremum concentration}
\label{ssec:global_optimality}

While the previous result establishes local stability, we further seek a global guarantee. Specifically, we aim to show that even if the global minimizer of the noisy landscape $Y$ (denoted $W_Y^*$) deviates from the true minimizer, it remains a high-utility approximation that captures the essential features of the underlying objective.

We define the set of $\mu$-suboptimal solutions as $\mathcal{N}_\mu = \{W : \PFL(W) - \PFL(W^\star) \le \mu\}$. A sufficient condition for the noisy minimizer $W_Y^\star$ to be in $\mathcal{N}_{\mu}$ depends on the maximum magnitude of the noise fluctuations over the entire search space $\Omega_k$.

\begin{lemma}[Supremum Bound]
    Let $\|\xi\|_\infty = \sup_{W \in \Omega_k} |\xi_W|$ be the supremum of the noise process. If $2\|\xi\|_\infty \le \mu$, then the noisy minimizer is $\mu$-suboptimal: $W_Y^\star \in \mathcal{N}_\mu$.
\end{lemma}
\begin{proof}
    By definition of the minimizers, $Y(W_Y^\star) \le Y(W^\star)$. Expanding this inequality using $Y(W) = \PFL(W) + \xi_W$:
    \begin{align*}
        \PFL(W_Y^\star) + \xi_{W_Y^\star} &\le \PFL(W^\star) + \xi_{W^\star} \nonumber \\
        \PFL(W_Y^\star) - \PFL(W^\star) &\le \xi_{W^\star} - \xi_{W_Y^\star}
    \end{align*}
    Since $\xi_{W^\star} - \xi_{W_Y^\star} \le |\xi_{W^\star}| + |\xi_{W_Y^\star}| \le 2 \|\xi\|_\infty$, we have the optimality gap $\PFL(W_Y^\star) - \PFL(W^\star) \le 2 \|\xi\|_\infty$. Thus, if $2\|\xi\|_\infty \le \mu$, the gap is bounded by $\mu$.
\end{proof}

Since $(\xi_W)_{W \in \Omega_k}$ is a centered Gaussian process with maximum variance $\sigma_{\PFL{}}^2$, we rely on the Borell-TIS inequality to bound the probability of large deviations.

\begin{theorem}{(Global $\mu$-Optimality):}
Let $\mathbb{E}[\|\xi\|_\infty]$ be the expected supremum of the noise process. For any $\mu > 2\mathbb{E}[\|\xi\|_\infty]$, the probability that the noisy minimizer is \emph{not} $\mu$-suboptimal satisfies:
\begin{equation*}
    \Pr(W_Y^\star \notin \mathcal{N}_\mu) \le 2 \exp\left( - \frac{(\mu/2 - \mathbb{E}[\|\xi\|_\infty])^2}{2\sigma_{\PFL{}}^2} \right)
\end{equation*}
\end{theorem}
\begin{proof}
    From the Lemma, the failure event $\{W_Y^\star \notin \mathcal{N}_\mu\}$ implies that the noise magnitude exceeded the threshold: $\|\xi\|_\infty > \mu/2$.
    
    The Borell-TIS inequality states that for a centered Gaussian process $f$ on a compact set with maximum variance $\sigma^2_{\max}$, the supremum deviates from its mean according to:
    \begin{equation*}
        \Pr\left(\left[\|\xi\|_\infty - \mathbb{E}[\sup_{W \in \Omega_k} (\xi_W)\right] > u\right) \le \exp\left(-\frac{u^2}{2\sigma^2_{\max}}\right)
    \end{equation*}
    We apply this to our process. Since we are bounding the absolute value $\|\xi\|_\infty = \sup | \xi_W |$, we apply the union bound to $\sup \xi_W$ and $\sup (-\xi_W)$, introducing a factor of 2.
    
    Let $u = \mu/2 - \mathbb{E}[\|\xi\|_\infty]$. Since we assume $\mu > 2\mathbb{E}[\|\xi\|_\infty]$, we have $u > 0$. Substituting into the tail bound:
    \begin{align*}
        \Pr(\|\xi\|_\infty > \mu/2) &= \Pr(\|\xi\|_\infty - \mathbb{E}[\|\xi\|_\infty] > u) \nonumber \\
        &\le 2 \exp\left( - \frac{u^2}{2\sigma_{\PFL{}}^2} \right) \nonumber \\
        &= 2 \exp\left( - \frac{(\mu/2 - \mathbb{E}[\|\xi\|_\infty])^2}{2\sigma_{\PFL{}}^2} \right)
    \end{align*}
    This completes the proof.
\end{proof}

This result establishes a squared exponential decay of the probability of $W^\star_Y$ from falling outside of $\mathcal{N}_{\mu}$. The tolerance margin $\mu$ improves this decay, while the privacy-induced noise slows it through the terms $\mathbb{E}[\|\xi\|_\infty]$ and $\sigma^2_{\PFL{}}$. In practice, setting the half tolerance margin so that it clears the expected noise supremum by $6$ standard deviations ($\mu/2 - \mathbb{E}[\|\xi\|_\infty] = 6\sigma_{\PFL{}}$), provides a failure probability below $10^{-7}$.

\section{Experiments}
\label{sec:exp}
Our experimental analysis is designed to assess the proposed framework along several axes:

\begin{enumerate}[leftmargin=*]
    \item \textbf{Computational practicality.}
    We evaluate whether the optimal federation search described in Algorithm~\ref{alg:main_approach} can be executed efficiently by the server in realistic settings.

    \item \textbf{Impact on vanilla federated learning.}
    We evaluate whether federations selected by minimizing \PFL{} yield:  (1) faster convergence, (2) improved predictive accuracy, and (3) inherently better group fairness, when training is performed using the standard \FedAvg{} algorithm.

    \item \textbf{Comparison with state-of-the-art methods.}
    We compare \PFL{}-optimized federations trained with conventional \FedAvg{} against suboptimal random federations trained using methods such as \FedProx{} and \Scaffold{}.
    The goal is to determine whether proactive federation selection alone can match or exceed the performance–fairness trade-offs achieved by these reactive approaches. Table \ref{tab:baselines} lists these methods.
\end{enumerate}

\paragraph*{\textbf{Folktables benchmarks and federation construction}}
To evaluate our proactive federation selection framework in a realistic and reproducible setting, we rely on the publicly available \emph{Folktables} benchmark suite \cite{ding2022retiringadultnewdatasets}, which is derived from the American Community Survey (ACS). Each benchmark includes demographic attributes that enable the evaluation of group fairness metrics. We focus on the gender sensitive attribute, and assess fairness using disparity measures alongside predictive performance. This experimental setting further supports proactive federation selection by providing (i) realistic non-IID data distributions across clients, (ii) fairness-sensitive prediction tasks grounded in socio-economic outcomes, and (iii) a sufficiently large and diverse client pool to evaluate combinatorial federation optimization under differential privacy constraints.

\paragraph{Federated partitioning via U.S. states}
We simulate a natural federated learning environment by treating each U.S.\ state as an individual client.  
Each client, therefore, holds data drawn from a distinct geographical and socio-economic distribution, inducing the statistical heterogeneity that characterizes real-world FL deployments.  
Candidate federations correspond to subsets of states, and the optimal federation search is performed over this combinatorial space.

\paragraph{Prediction tasks}
Our experiments cover all major Folktables classification benchmarks, each corresponding to a different socio-economic prediction problem:

\begin{itemize}[leftmargin=*]
    \item \textbf{ACSIncome:} Predicts whether an individual’s annual income exceeds a fixed threshold.  
    This task is widely used to study economic inequality and algorithmic fairness in income prediction.

    \item \textbf{ACSEmployment:} Predicts employment status, capturing labor-market participation patterns across demographic groups and regions.

    \item \textbf{ACSPublicCoverage:} Predicts whether an individual is covered by public health insurance programs, providing a proxy for socio-economic vulnerability and healthcare access.

    \item \textbf{ACSTravelTime:} Predicts whether the commute time exceeds a specified duration, capturing transportation inequality and regional infrastructure differences.
\end{itemize}

\begin{table}[t]
%\small
\centering
\caption{Overview of our considered baselines.}
\label{tab:baselines}
\begin{tabular}{@{}lp{6.5cm}@{}}
\toprule
\textbf{Method} & \textbf{Principle} \\
\midrule
\multicolumn{2}{@{}l}{\textit{Aggregation and/or local regularization Strategies}} \\
\midrule
FedAvg \cite{mcmahan2017fedavg} & Weighted average of client models proportional to dataset sizes. \\
FedProx \cite{li2020federated} & Proximal regularization term to handle client heterogeneity. \\
SCAFFOLD \cite{karimireddy2021scaffold}& Control variates to correct client drift from non-IID data. \\
\midrule
\multicolumn{2}{@{}l}{\textit{Client Selection Strategies}} \\
\midrule
Random \cite{mcmahan2017fedavg} & Uniform random client sampling. \\
UCB-CS \cite{cho2020bandit} & Bandit-based UCB scores balancing exploration/exploitation. \\
Threshold \cite{ribero2020communication} & Selects high-loss clients; O-U process for stale estimates. \\
FedSampling \cite{qi2023fedsampling} & Selection probability proportional to client data size. \\
\bottomrule
\end{tabular}
\end{table}

\subsection{Optimal federation search cost}
For each federation size, $k \in \{5, 10, 15\}$, and each of the four ACS classification tasks, we perform $5$ independent runs following the methodology outlined in Algorithm \ref{alg:main_approach}. Our simulated annealing parameters are provided in Table \ref{tab:sa_hyperparams}.

\begin{table}[h]
\centering
\caption{Simulated Annealing parameters.}
\label{tab:sa_hyperparams}
\renewcommand{\arraystretch}{1.2}
\begin{tabular}{lc}
\hline
\textbf{Parameter} & \textbf{Value} \\ \hline
Initial temperature ($\tau_0$) & 1.0 \\
Cooling rate ($\eta$) & 0.98 \\
Minimum temperature ($\tau_{\min}$) & $10^{-4}$ \\
Maximum iterations & 5000 \\
Iterations per temperature & 15 \\ \hline
\end{tabular}
\end{table}

Our simulated annealing configuration is designed to balance exploration and convergence. An initial temperature of $\tau_0=1.0$ aligns with the PFL scale, enabling broad early exploration by accepting many worse moves. A slow geometric cooling rate ($\alpha = 0.98$) creates about 340 temperature levels, allowing a gradual transition from exploration to exploitation, with 15 neighbor evaluations per level for efficient local search ($\approx 5,100$ total evaluations). The process stops when the temperature drops below $10^{-4}$ or the maximum iterations are reached, ensuring convergence to a robust optimum.
Table~\ref{tab:sa_results} reports the empirical runtime across different federation sizes and ACS tasks. All experiments were conducted on an Intel Core i7-12700H processor (12th generation, 14 cores) and $22$GB of RAM, running Ubuntu 22.04 LTS. Notably, our simulated annealing implementation is sequential. Advanced parallel implementations induce significant speedups as reported by Greening \etal{} \cite{GREENING1990293}. For page limitation purposes, we report in the main body the results for $k=5$, and provide the remaining figures in Appendix \ref{sec:appendix_additional_exp}.

\begin{table*}[t]
\centering
\caption{Optimal federations found by simulated annealing for each ACS task and federation size $k$ (best over 5 runs). Noise scale $\sigma$ is computed via RDP composition for $\epsilon=1.0$, $\delta=10^{-5}$. Search space sizes are: $\binom{50}{5} \approx 2.1 \times 10^6$, $\binom{50}{10} \approx 1.0 \times 10^{10}$, and $\binom{50}{15} \approx 2.3 \times 10^{12}$.}
\label{tab:sa_results}
\footnotesize
\begin{tabular}{ll|cccc}
\toprule
\textbf{Task} & $k$ & \textbf{PFL} & $\sigma$ & \textbf{Runtime (min)} & \textbf{Optimal federations (states)} \\
\midrule
\multirow{3}{*}{\shortstack[l]{ACSIncome\\(10 feature)}} 
  & 5  & $-0.167 \pm 0.01$ & \multirow{3}{*}{39.81} & 23 & AK, CA, CT, MD, RI \\
  & 10 & $-0.192 \pm 0.01$ &  & 43 & AZ, CA, DE, MA, MD, NH, NJ, NV, OR, VA \\
  & 15 & $-0.187 \pm 0.02$ &  & 69 & AK, AZ, CA, DE, MD, NH, NJ, NM, NV, OR, RI, SD, VA, VT, WA \\
\midrule
\multirow{3}{*}{\shortstack[l]{ACSEmployment\\(16 feature)}} 
  & 5  & $-0.715 \pm 0.03$ & \multirow{3}{*}{60.62} & 27 & CO, MN, ND, NE, UT \\
  & 10 & $-0.578\pm 0.02$ &  & 45 & CO, LA, MN, NE, NM, OH, RI, SD, WI, WY \\
  & 15 & $-0.565\pm 0.02$ &  & 67 & AK, CO, CT, IA, IL, LA, MA, MN, MS, MT, ND, RI, VT, WI, WY \\
\midrule
\multirow{3}{*}{\shortstack[l]{ACSPublicCov.\\(19 feature)}} 
  & 5  & $-0.809 \pm 0.04$ & \multirow{3}{*}{71.02} & 31  & AR, KY, MO, MS, SC \\
  & 10 & $-0.434 \pm 0.03$ &  &  48 & AK, AL, AR, HI, KY, LA, ME, MO, MT, SC \\
  & 15 & $-0.0214\pm 0.03$ &  & 75    & AK, AR, CO, KY, MI, MO, MS, MT, NC, NM, OH, PA, SC, WI, WV \\
\midrule

\multirow{3}{*}{\shortstack[l]{ACSTravelTime\\(16 feature)}} 
  & 5  & $-0.282\pm 0.02$ & \multirow{3}{*}{60.62} & 31 & CA, FL, HI, NY, SD \\
  & 10 & $-0.295\pm 0.02$ &  & 46 & AK, CA, CO, FL, IL, MA, MS, ND, NM, NY \\
  & 15 & $-0.230\pm 0.03$ &  & 72 & CA, CT, FL, GA, IL, MA, MD, MT, NV, NY, VA, VT, WA, WV, WY \\
\bottomrule
\end{tabular}
\end{table*}

\textbf{Findings.} The empirical runtimes reported in Table \ref{tab:sa_results} confirm that the proposed simulated annealing configuration remains computationally tractable as the federation size increases. While the runtime grows with $k$, the increase is sub-quadratic and remains within practical limits, ranging from approximately $23$ minutes for $k=5$ to 75 minutes for $k=15$. The low \PFL{}'s standard deviations indicate stable execution outcomes across the 5 runs, suggesting that the annealing schedule and fixed evaluation budget provide consistent optimization behavior despite DP noise. Overall, these results demonstrate that the optimization procedure scales reasonably with federation size while preserving manageable runtime. Finally, as observed in Section \ref{sec:noisy_sa}, the differential-privacy noise needed to maintain a consistent $\epsilon=1.0$ level grows linearly with the number of features. This is due to the larger number of queries, which induces larger RDP composition results. Nevertheless, the \PFL{}'s deviations remained stable across different benchmarks, regardless of the number of features. This is due to the simulated annealing stability mechanisms described in Section \ref{sec:noisy_sa}.

\begin{figure*}[t]
    \centering
    \begin{subfigure}{0.40\linewidth}
        \centering
        \includegraphics[width=\linewidth]{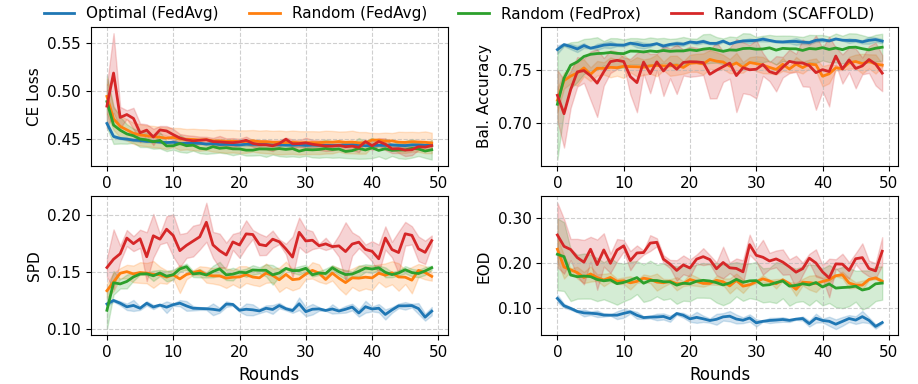}
        \caption{ACSIncome}
        \label{fig:exp1_income}
    \end{subfigure}
    \hspace{0.5cm}
    \begin{subfigure}{0.40\linewidth}
        \centering
        \includegraphics[width=\linewidth]{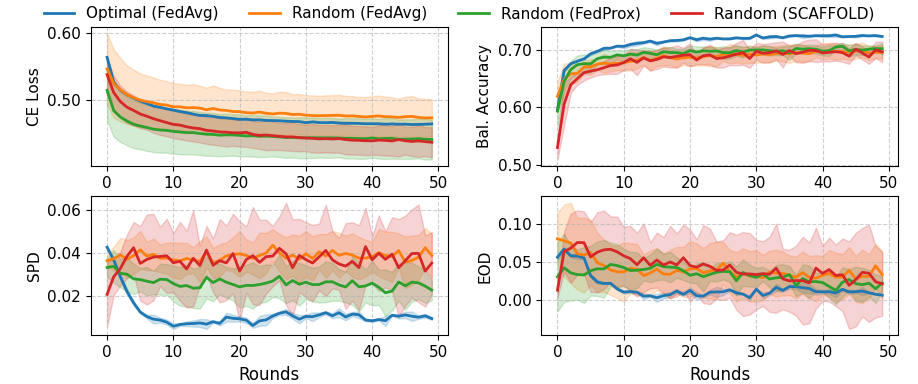}
        \caption{PublicCoverage}
        \label{fig:exp1_public_coverage}
    \end{subfigure}

    \vspace{0.5em}

    \begin{subfigure}{0.40\linewidth}
        \centering
        \includegraphics[width=\linewidth]{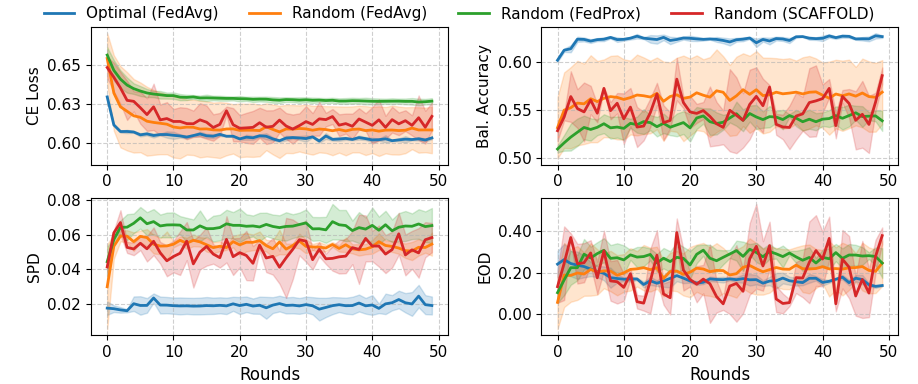}
        \caption{TravelTime}
        \label{fig:exp1_travel_time}
    \end{subfigure}
    \hspace{0.5cm}
    \begin{subfigure}{0.40\linewidth}
        \centering
        \includegraphics[width=\linewidth]{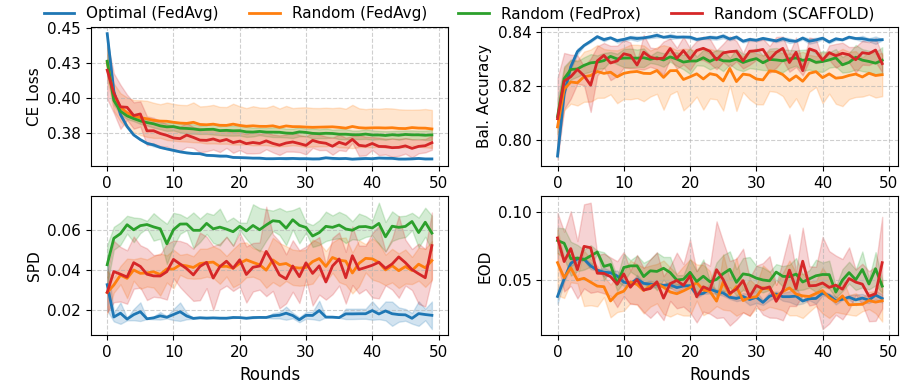}
        \caption{ACSEmployment}
        \label{fig:exp1_employment}
    \end{subfigure}

    \caption{Convergence and fairness of our optimal federations compared with random ones, with FedAvg, FedProx, and SCAFFOLD, for ACSIncome, ACSEmployment, ACSPublicCoverage, and ACSTravelTime.}
    \label{fig:exp_1_k_5}
\end{figure*}

\begin{figure*}[!htb]
    \centering
    \begin{subfigure}{0.40\linewidth}
        \centering
        \includegraphics[width=\linewidth]{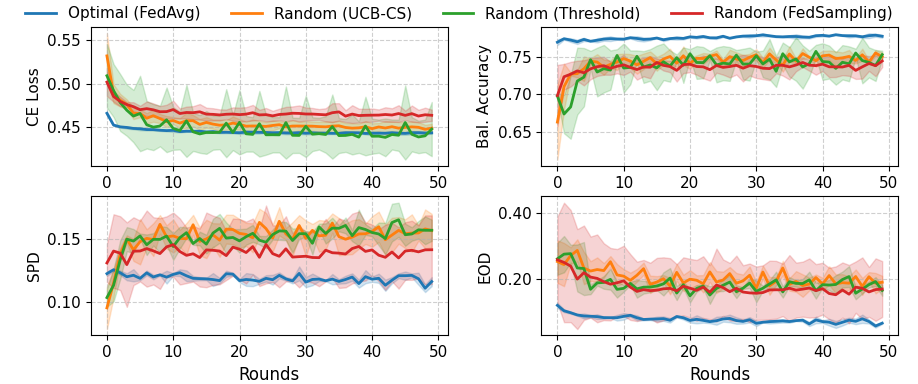}
        \caption{ACSIncome}
        \label{fig:exp2_income}
    \end{subfigure}
    \hspace{0.5cm}
    \begin{subfigure}{0.40\linewidth}
        \centering
        \includegraphics[width=\linewidth]{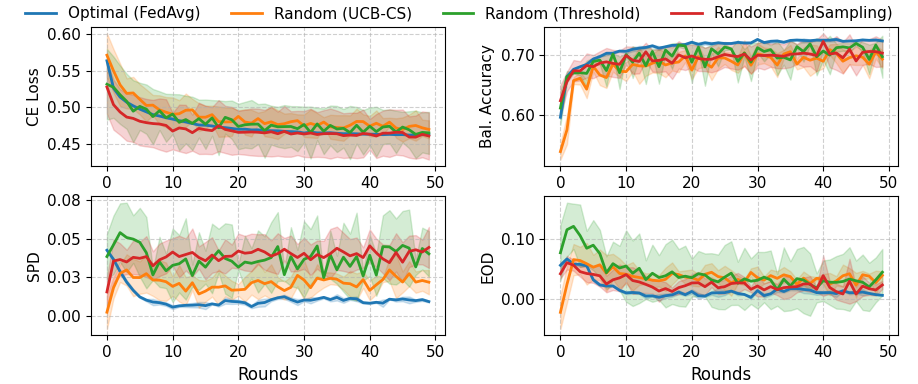}
        \caption{PublicCoverage}
        \label{fig:exp2_public_coverage}
    \end{subfigure}

    \begin{subfigure}{0.40\linewidth}
        \centering
        \includegraphics[width=\linewidth]{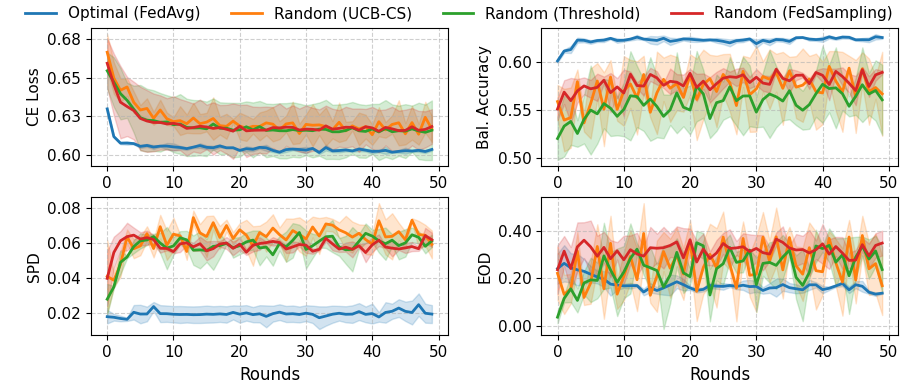}
        \caption{TravelTime}
        \label{fig:exp2_travel_time}
    \end{subfigure}
    \hspace{0.5cm}
    \begin{subfigure}{0.40\linewidth}
        \centering
        \includegraphics[width=\linewidth]{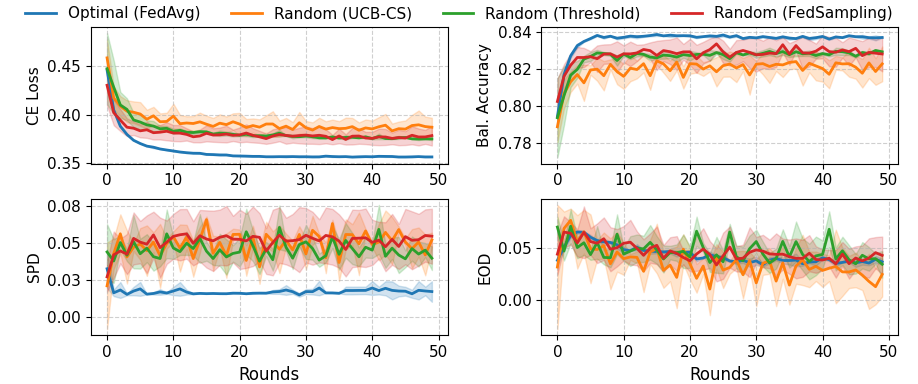}
        \caption{ACSEmployment}
        \label{fig:exp2_employment}
    \end{subfigure}

    \caption{Convergence and fairness of our optimal federations compared with random ones, with sampling strategies: UCB-CS \cite{cho2020bandit}, FedSampling \cite{qi2023fedsampling}, and Threshold-based participation \cite{ribero2020communication}, for ACSIncome, ACSEmployment, ACSPublicCoverage, and ACSTravelTime with $k=5$.}
    \label{fig:exp_2_k_5}
\end{figure*}

\subsection{Comparison with adaptive aggregations}

We compare optimally found federations (outlined in Tables \ref{tab:sa_results}) trained with standard \FedAvg{} with randomly sampled federations trained using SCAFFOLD ~\cite{karimireddy2021scaffold} and FedProx~\cite{li2020federated}. These methods are explicitly designed to mitigate client drift and improve convergence under strong local data heterogeneity. Our goal is to evaluate whether structural optimization of federations outperforms uniformly sampled federations trained with advanced aggregation techniques tailored for non-IID settings. Results are depicted in Figures \ref{fig:exp_1_k_5} and \ref{fig:exp_2_k_5}. 

\textbf{Findings.}
From a predictive performance perspective, methods such as FedProx and SCAFFOLD improve convergence stability under uniform random client sampling compared to standard \FedAvg{}. However, in our experiments, these approaches do not outperform our optimally constructed federations trained with vanilla \FedAvg{}.

This observation is consistent with the underlying mechanisms of these algorithms. Both \FedProx{} and \Scaffold{} mitigate client drift through corrective mechanisms, respectively, a proximal regularization and variance-reduction control variates, that penalize large deviations from the global objective. While these mechanisms enhance optimization stability in heterogeneous environments, they operate reactively by constraining local updates during training.
In contrast, our approach acts at the federation design level. By selecting clients that exhibit globally consistent correlation structures, we construct federations in which local objectives are aligned with the global objective. As a result, the aggregated updates exhibit limited drift without requiring additional regularization terms. This proactive structural alignment allows unconstrained local progress while maintaining global consistency. This, in turn, preserves utility without introducing optimization penalties.

From a fairness perspective, as expected FedProx and SCAFFOLD do not have a clear impact on group fairness, since they are not primarily designed for this purpose. Therefore, our optimal federation with \FedAvg{} outperforms both following the EOD and SPD measures.

\subsection{Comparison with Reactive Sampling Strategies.}
We compare optimal federations trained with \FedAvg{} against uniformly sampled federations trained with three reactive client-sampling methods: ClusterFL~\cite{Ouyang2021ClusterFL}, UCB-CS~\cite{cho2020bandit}, and FedSampling~\cite{qi2023fedsampling}. These methods implement biased client-selection strategies aimed at accelerating convergence and reducing the final optimization loss. Specifically, ClusterFL and UCB-CS increase the selection probability of clients exhibiting higher local losses, while FedSampling prioritizes clients with larger local datasets. The underlying intuition is that emphasizing such clients focuses the global optimization process on either underperforming local objectives or participants with greater learning capacity, thereby improving training efficiency and reducing the global objective more rapidly. Figures~\ref{fig:exp_1_k_5} and~\ref{fig:exp_2_k_5} report the corresponding results.

\textbf{Findings.}
From a predictive performance perspective, reactive sampling strategies such as ClusterFL, UCB-CS, and FedSampling improve convergence speed compared to uniform random sampling with \FedAvg{}. By prioritizing clients with higher losses or larger datasets, these approaches concentrate training effort on clients that either contribute more information or exhibit larger optimization gaps. These targeted selections accelerate early-stage learning and reduce the overall training loss.

However, in our experiments, these reactive methods still do not outperform the performance achieved by our optimally constructed federations trained with vanilla \FedAvg{}. While biased sampling can partially compensate for heterogeneity by dynamically adjusting the client selection probabilities, it remains fundamentally reactive. The server must continuously observe client behavior (\eg{}, loss values or dataset sizes) and adapt the sampling distribution during training. In contrast, our approach addresses heterogeneity at the federation construction stage. By proactively selecting clients whose data distributions exhibit strong correlation structures, the resulting federation naturally induces local objectives that are aligned with the global objective. Consequently, the aggregated updates exhibit limited conflict, allowing stable and efficient optimization without requiring dynamic sampling biases.

Regarding group fairness, since these methods are not specifically designed to address fairness, they exhibit unstable and sporadic behavior throughout training (in terms of SPD and EOD), and almost consistently perform worse than our optimally constructed federations trained with \FedAvg{}.

\section{Conclusion}
This work suggests and demonstrates that proactive federation search constitutes a promising and largely unexplored design dimension for federated learning systems, providing significant improvement potential along the efficiency, privacy, and fairness axes. Specifically, we show that differentially private statistical summaries combined with an offline optimization phase performed at the server side enable the construction of optimal federations with respect to utility and group fairness, with a limited privacy exposure. 

Our perspectives involve extending our approach to image data using semantic feature extractors and addressing related federated data valuation problems, such as finding the smallest statistically representative federations among larger pools of candidate clients for efficient training with limited privacy exposure.

%% file: appendices.tex
\section{Group fairness metrics}
\label{sec:appendix_fairness_metrics}

Group fairness aims to ensure that a model's predictions do not exhibit systematic disparities across groups defined by a sensitive attribute $S$ (\eg{}, sex or race). Let $T$ denote the target label and $\hat{Y}$ the model prediction.

A central notion is \emph{Demographic Parity} (DP), which requires independence between predictions and the sensitive attribute:
\begin{equation*}
\Pr(\hat{Y} = 1 \mid S = s_1) = \Pr(\hat{Y} = 1 \mid S = s_2), \quad \forall s_1, s_2.
\end{equation*}
Violations of DP are typically quantified via the \emph{Statistical Parity Difference} (SPD), defined as the difference in positive prediction rates between groups.

Another widely used notion is \emph{Equal Opportunity} (EO), which requires equal true positive rates across groups:
\begin{equation*}
\Pr(\hat{Y} = 1 \mid T = 1, S = s_1) = \Pr(\hat{Y} = 1 \mid T = 1, S = s_2)
\end{equation*}
Disparities are measured using the \emph{Equal Opportunity Difference} (EOD), defined as the difference in true positive rates between groups.

More generally, group fairness can be viewed as limiting the dependence between $S$ and either the predictions $\hat{Y}$ or the target $T$, conditioned on relevant variables. In this work, we adopt an information-theoretic perspective: fairness violations arise when the sensitive attribute $S$ carries significant information about the target $Y$ or about non-sensitive features used for prediction. This includes both:
(i) \emph{direct bias}, when $S$ is informative of $T$, and 
(ii) \emph{indirect (proxy) bias}, when $S$ is correlated with non-sensitive features that influence $\hat{Y}$.

Accordingly, we evaluate fairness using disparity metrics such as SPD and EOD, while modeling fairness risks through mutual information terms that capture both direct and indirect dependencies between $S$, $T$, and the feature set.

\section{Empirical privacy auditing through membership inference attack}
\label{sec:appendix_dp_audit}
To empirically validate the privacy guarantee at $\epsilon_1 = 1.0$ of our optimal selection framework, we consider the optimal membership inference auditing using hypothesis testing, and the Neyman-Pearson lemma \cite{neyman2026ontheproblem}.
Given a target record $r^\star$, we test the hypothesis:
$$
H_0: r^\star \notin D_i
\quad \text{vs.} \quad
H_1: r^\star \in D_i
$$

Under the Gaussian mechanism, each noisy contingency table cell $\tilde{N}^{X, Y}(x, y)$ corresponding to $r^\star$'s attribute values for $X$ and $Y$ follows:
$$
\tilde{N}^{X, Y}(x, y) \sim
\begin{cases}
\mathcal{N}({N}_{-r^\star}(x, y), \sigma^2) \text{ \,\, if \,\,} & H_0 \\
\mathcal{N}({N}_{-r^\star}(x, y) + 1, \sigma^2) \text{ \,\, if \,\,} & H_1
\end{cases}
$$
where ${N}_{-r^\star}(x, y)$ denotes the exact count of samples with $X=x$, and $Y=y$ in the dataset, without $r^\star$, and is assumed to be known by the attacker (server), along with the differential-privacy noise scale $\sigma$.

The total log-likelihood ratio across all pairs of features $\mathcal{T}$ for which the attacker observes released noisy tables is:
$$
\Lambda_{\text{joint}}(r^\star)
=
\sum_{(X,Y)\in\mathcal{T}}
\log \left(\frac{
\Pr\!\left(\tilde{N}^{X,Y}(x,y)\mid H_1\right)
}{
\Pr\!\left(\tilde{N}^{X,Y}(x,y)\mid H_0\right)
}\right)
$$
Using the Gaussian PDFs, simplifying, and summing over all feature pairs $(X,Y)\in\mathcal{T}$ gives the closed-form joint statistic:
$$
\Lambda_{\text{joint}}(r^\star)
=
\frac{1}{\sigma^2}
\sum_{(X,Y)\in\mathcal{T}}
\left[
\tilde{N}^{X,Y}(x,y)
-
N_{-r^\star}^{X,Y}(x,y)
-
\frac{1}{2}
\right]
$$
The Neyman-Pearson lemma \cite{neyman2026ontheproblem} provides that the optimal decision rule is the test:
$\left[\Lambda_{\text{joint}}(r^\star) > \tau_\alpha \right]$,
where $\tau_\alpha$ is chosen to achieve a target false positive rate $\alpha$.

\subsection*{Empirical validation}
We evaluate the Log-Likelihood Ratio Test (LRT) attack for multiple privacy budgets
$\epsilon_1 \in \{0.1, 0.5, 1.0, 2.0, 5.0, +\infty\}$, with a fixed $\delta = 10^{-5}$. The noise scales used for each ACS task are outlined in Table \ref{tab:sa_results}.
For each benchmark, the noise scale $\sigma$ is calibrated via RDP composition across all released contingency tables. For each task, the audit is performed using $25$k records as the dataset, and $1000$ target records ($r^\star$) uniformly sampled across all states.
We assess the attack using:
(i) the Area Under the ROC Curve (AUC), and
(ii) the True Positive Rate (TPR) at a fixed False Positive Rate of $\mathrm{FPR}=1\%$.
These metrics are computed both in the single-table setting and in the multi-table setting,
highlighting the increased inference power of an adversary with access to multiple noisy tables.
The results are reported in Figure~\ref{fig:dp_audit}.

\begin{figure}[h]
    \centering
    \begin{subfigure}{0.80\linewidth}
        \centering
        \includegraphics[width=\linewidth]{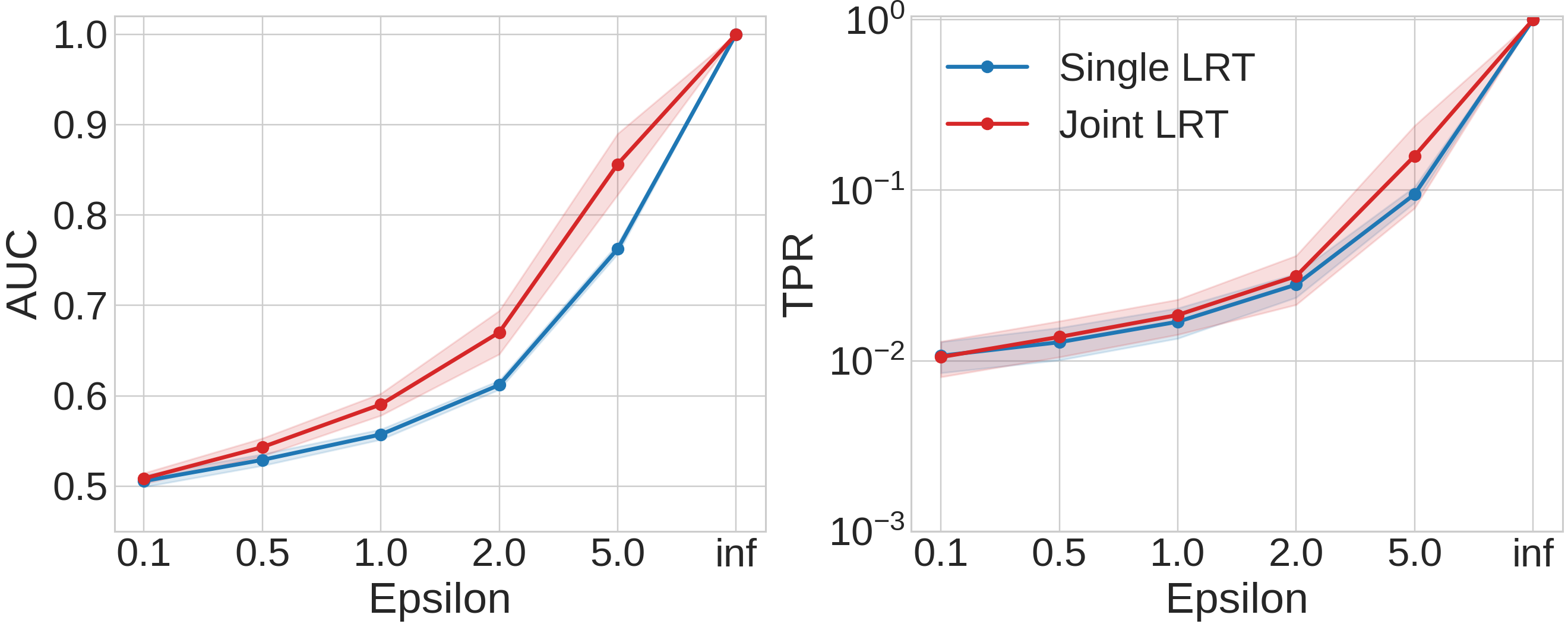}
        \caption{ACSIncome}
        \label{fig:exp2_income}
    \end{subfigure}
    
    \begin{subfigure}{0.80\linewidth}
        \centering
        \includegraphics[width=\linewidth]{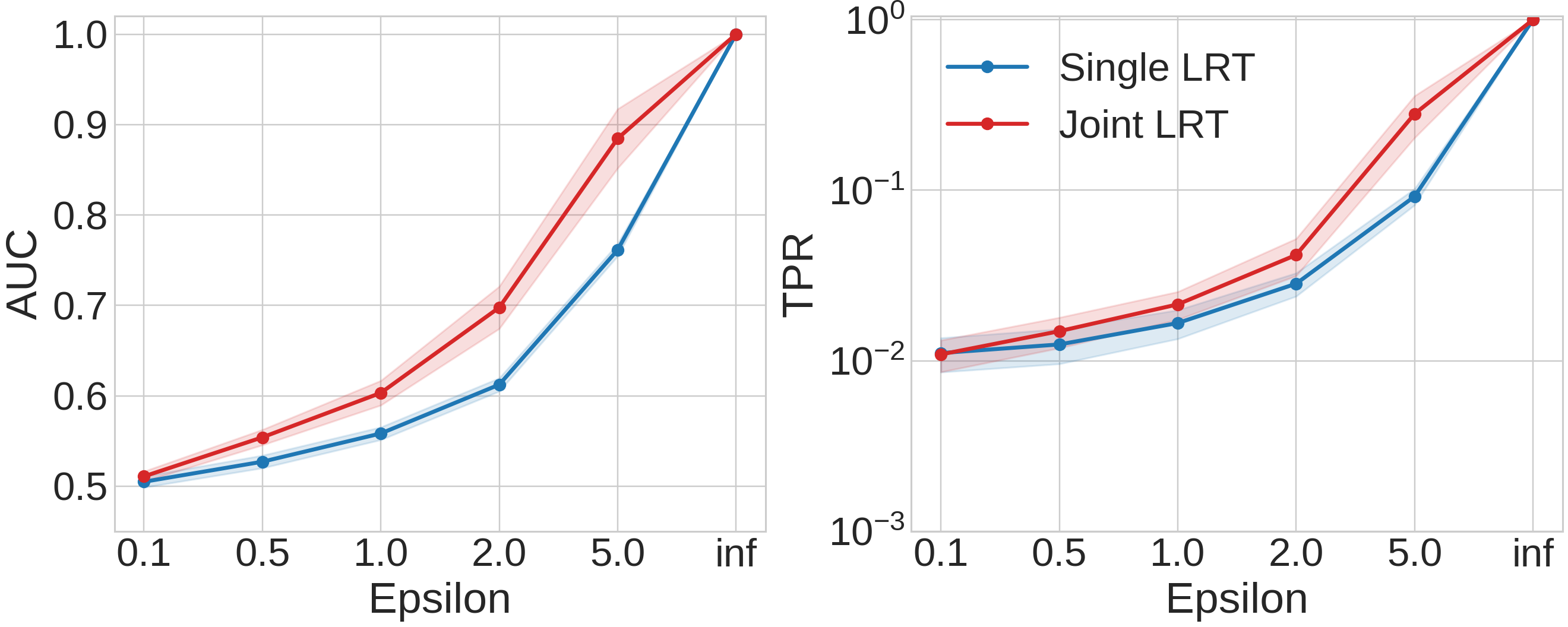}
        \caption{ACSPublicCoverage}
        \label{fig:exp2_public_coverage}
    \end{subfigure}

    \begin{subfigure}{0.80\linewidth}
        \centering
        \includegraphics[width=\linewidth]{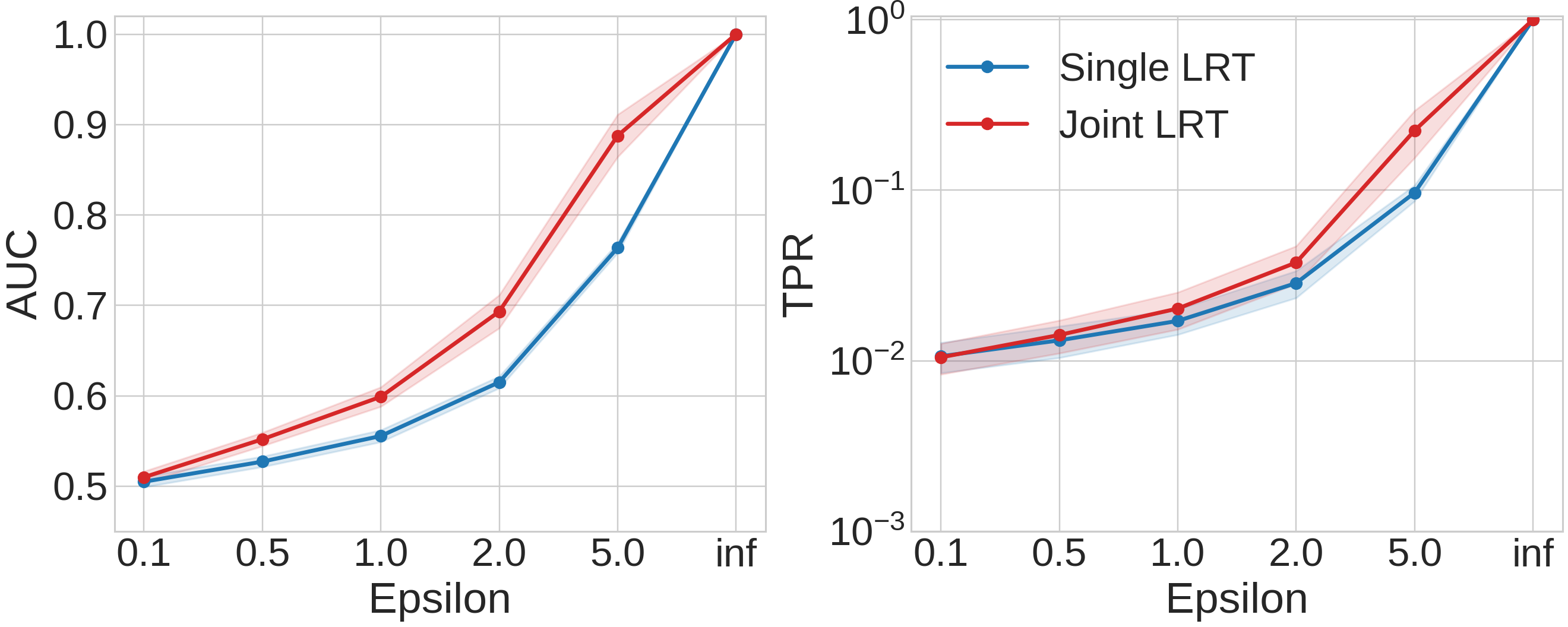}
        \caption{ACSTravelTime}
        \label{fig:exp2_travel_time}
    \end{subfigure}
    
    \begin{subfigure}{0.80\linewidth}
        \centering
        \includegraphics[width=\linewidth]{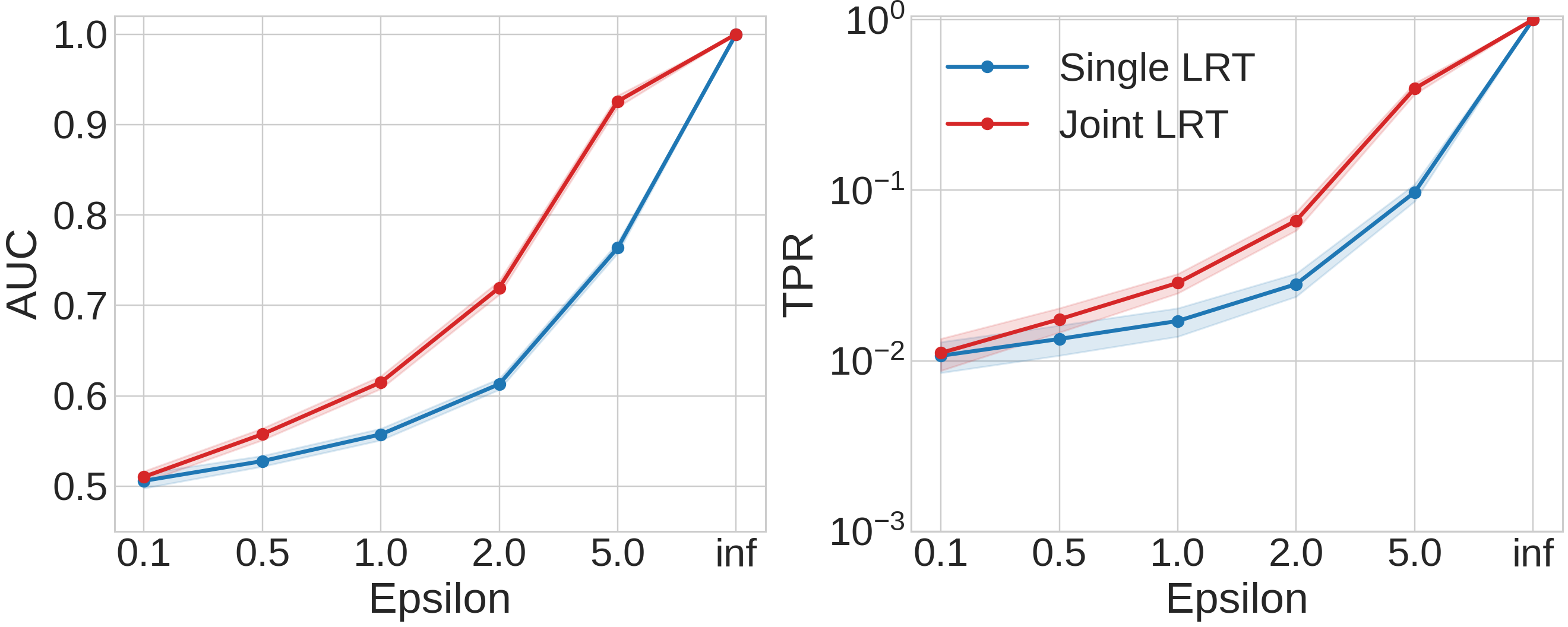}
        \caption{ACSEmployment}
        \label{fig:exp2_employment}
    \end{subfigure}

    \caption{Privacy auditing using Log-likelihood Ratio Test (LRT) for ACSIncome, ACSPublicCoverage, ACSTravelTime, and ACSEmployment, from both a single and multiple contingency tables observation.}
    \label{fig:dp_audit}
\end{figure}

\paragraph*{Findings}
Figure \ref{fig:dp_audit} shows that both single and multiple tables LRT performances (AUC and TPR) fall as low as a random guess at $\epsilon_1 = 0.1$, and achieve perfect performance for $\epsilon_1 = \infty$. The steepest attack improvements occur in the regime $\epsilon_1 \in [1.0, 5.0]$ where the noise standard deviation drops significantly faster due to the squared privacy budget denominator in the noise variance calibration $(\sigma^2 \propto 1/\epsilon^2)$. We therefore adopt $\epsilon_1 = 1.0$ 
as our operating point. Indeed, it lies at the privacy-utility sweet-spot where the noise is large enough to keep the attack close to random guessing (at most an AUC of $0.6$ is observed for ACSEmployment and ACSTravelTime), yet the privacy budget remains 
practically meaningful for our optimal federation search.

\section{Additional experiments}
\label{sec:appendix_additional_exp}
In this section, we present the results of the training of federations of size $10$ and $15$ with a similar experimental setting as presented in Section \ref{sec:exp}. Results are reported in Figures \ref{fig:exp_1_k_10}, and \ref{fig:exp_2_k_10} for federations of size $10$, and Figures \ref{fig:exp_1_k_15} and \ref{fig:exp_2_k_15} report the results for federations of size $15$.

\paragraph*{Findings}
The additional experiments for federation sizes $(k=5)$ and $(k=15)$ provide a complementary perspective on how the proposed proactive selection mechanism behaves across different regimes of aggregation, without altering the core trends observed in the main text. For smaller federations $(k=5)$, the results show that suboptimal or random choices lead to highly variable performance and fairness outcomes even when reactive aggregation or sampling strategies are employed, whereas the \PFL{}-driven selection consistently identifies compact federations that preserve strong predictive signals while avoiding concentrated sources of bias. This emphasizes the sensitivity of small-scale federations to statistical heterogeneity and reinforces the benefit of explicitly optimizing cross-feature dependencies prior to training. In contrast, for larger federations $(k=15)$, the results illustrate a smoothing effect induced by aggregation, where the increased number of clients naturally reduces variance in both utility and fairness metrics. Nevertheless, even in this more stable regime, the proactive selection strategy maintains a consistent advantage over random or reactive baselines, showing that the gains are not solely due to variance reduction but stem from a more favorable alignment of data distributions within the optimally found federations. Notably, the performance and fairness gaps persist, which aligns with the $k=5$ experiments. These experiments confirm that the proposed framework is effective across both low and high-cardinality federation settings.

\begin{figure}[h]
    \centering
    \begin{subfigure}{0.85\linewidth}
        \centering
        \includegraphics[width=\linewidth]{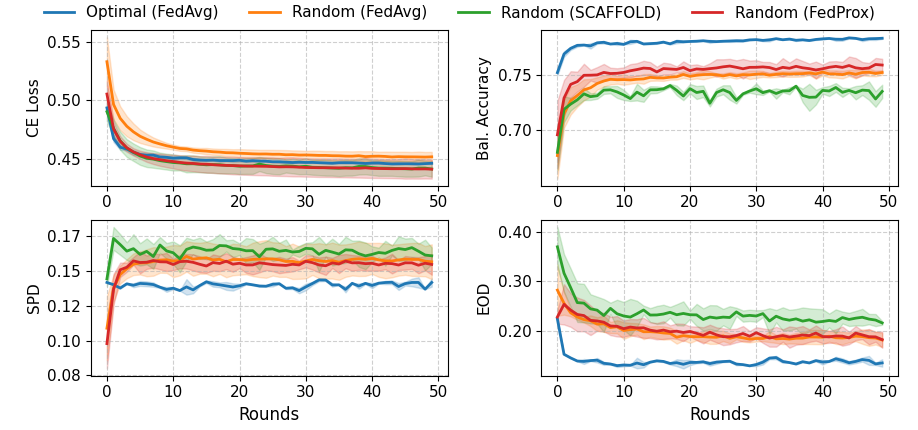}
        \caption{ACSIncome}
        \label{fig:exp2_income}
    \end{subfigure}

    \begin{subfigure}{0.85\linewidth}
        \centering
        \includegraphics[width=\linewidth]{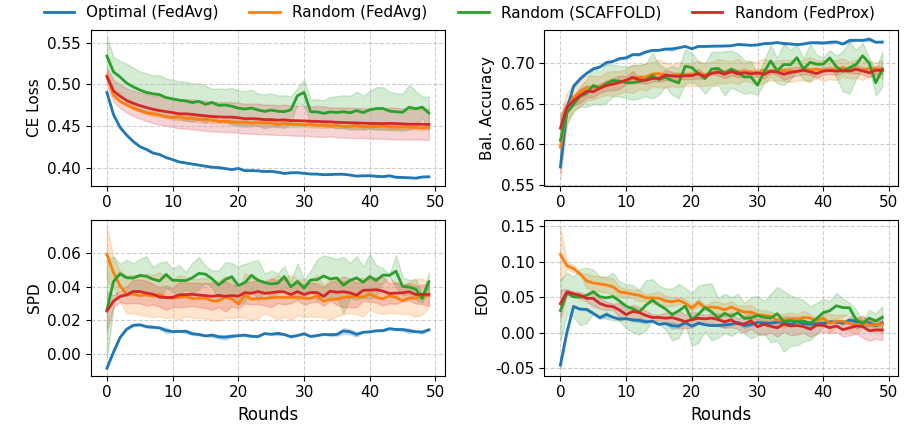}
        \caption{PublicCoverage}
        \label{fig:exp2_public_coverage}
    \end{subfigure}

    \begin{subfigure}{0.85\linewidth}
        \centering
        \includegraphics[width=\linewidth]{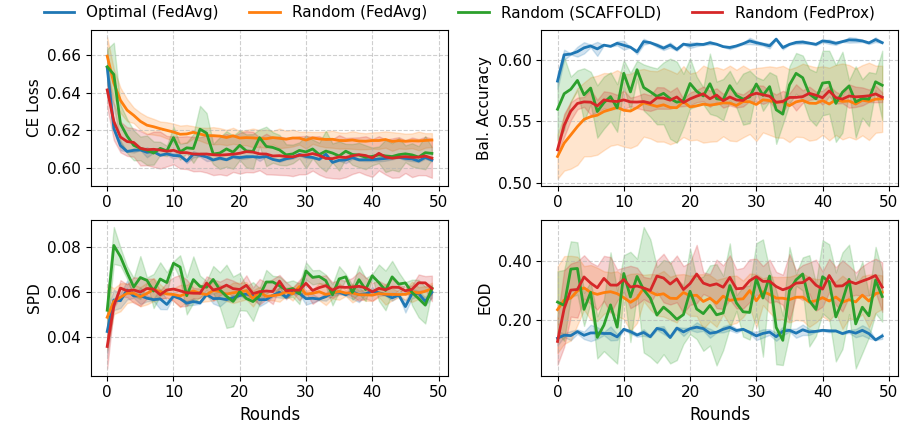}
        \caption{TravelTime}
        \label{fig:exp2_travel_time}
    \end{subfigure}

    \begin{subfigure}{0.85\linewidth}
        \centering
        \includegraphics[width=\linewidth]{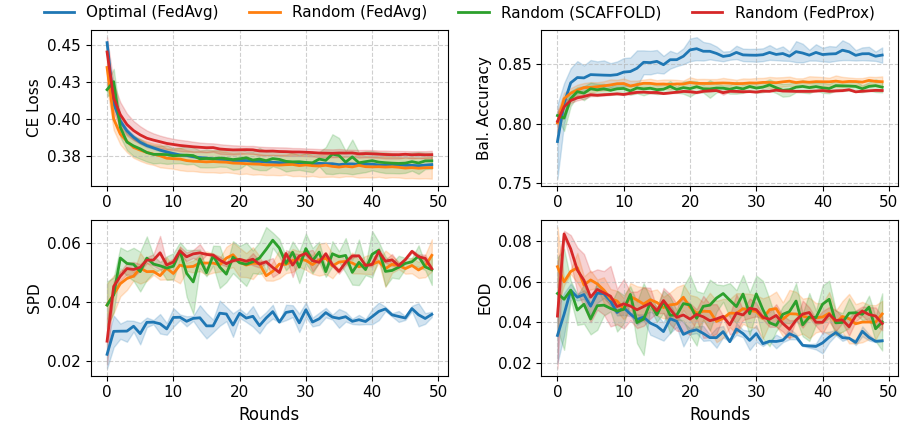}
        \caption{ACSEmployment}
        \label{fig:exp2_employment}
    \end{subfigure}

    \caption{Convergence and fairness of our optimal federations compared with random ones, with FedAvg, FedProx, and SCAFFOLD, for ACSIncome, ACSEmployment, ACSPublicCoverage, and ACSTravelTime with $k=10$.}
    \label{fig:exp_1_k_10}
\end{figure}

\begin{figure}[h]
    \centering
    \begin{subfigure}{0.90\linewidth}
        \centering
        \includegraphics[width=\linewidth]{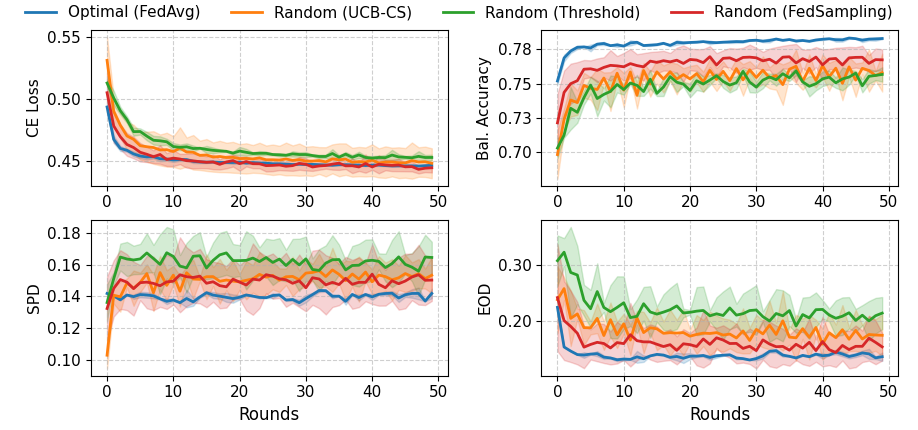}
        \caption{ACSIncome}
        \label{fig:exp2_income}
    \end{subfigure}

    \begin{subfigure}{0.90\linewidth}
        \centering
        \includegraphics[width=\linewidth]{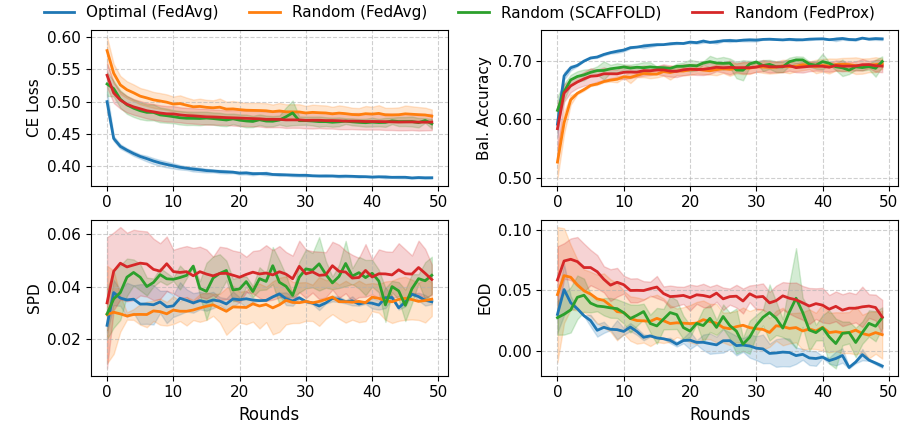}
        \caption{PublicCoverage}
        \label{fig:exp2_public_coverage}
    \end{subfigure}

    \begin{subfigure}{0.90\linewidth}
        \centering
        \includegraphics[width=\linewidth]{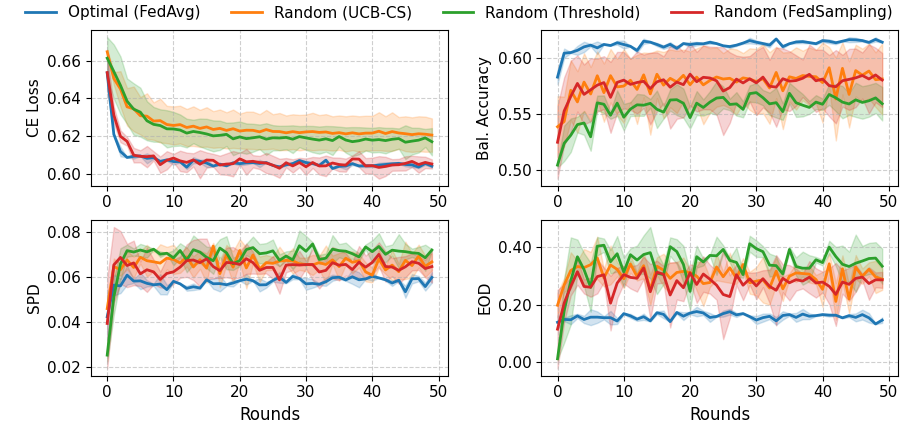}
        \caption{TravelTime}
        \label{fig:exp2_travel_time}
    \end{subfigure}

    \begin{subfigure}{0.90\linewidth}
        \centering
        \includegraphics[width=\linewidth]{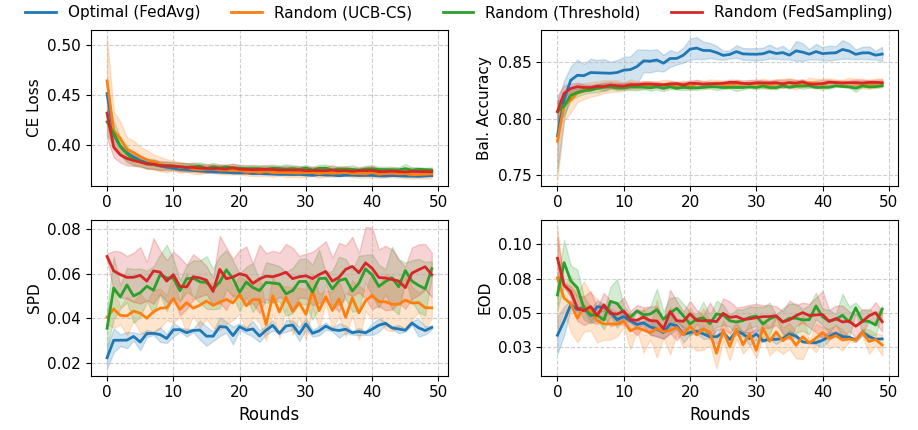}
        \caption{ACSEmployment}
        \label{fig:exp2_employment}
    \end{subfigure}

    \caption{Convergence and fairness of our optimal federations compared with random ones, with sampling strategies: UCB-CS \cite{cho2020bandit}, FedSampling \cite{qi2023fedsampling}, and Threshold-based participation \cite{ribero2020communication}, for ACSIncome, ACSEmployment, ACSPublicCoverage, and ACSTravelTime with $k=10$.}
    \label{fig:exp_2_k_10}
\end{figure}

\begin{figure*}[h]
    \centering
    \begin{subfigure}{0.45\linewidth}
        \centering
        \includegraphics[width=\linewidth]{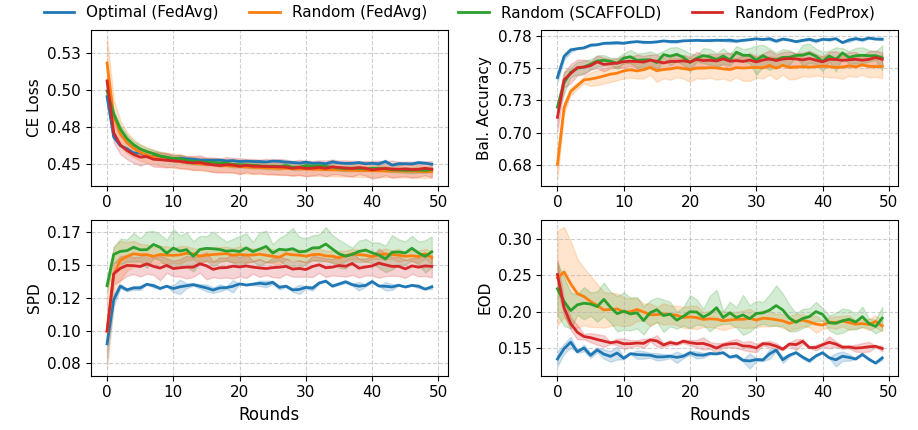}
        \caption{ACSIncome}
        \label{fig:exp2_income}
    \end{subfigure}
    \hspace{0.5cm}
    \begin{subfigure}{0.45\linewidth}
        \centering
        \includegraphics[width=\linewidth]{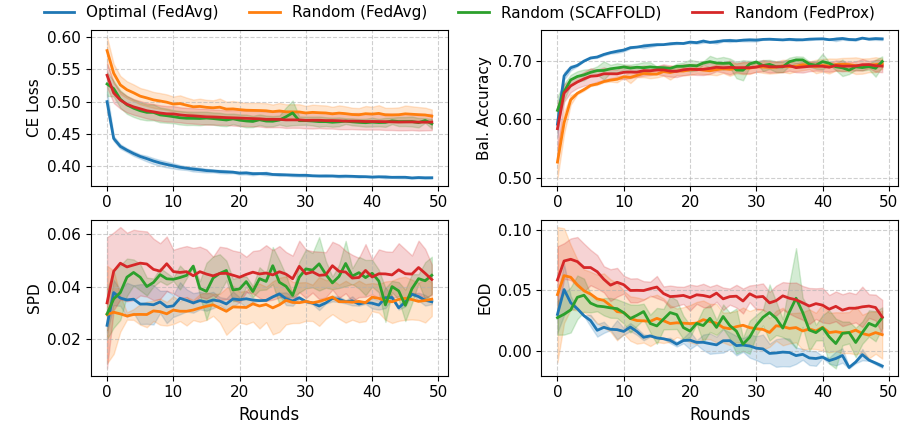}
        \caption{PublicCoverage}
        \label{fig:exp2_public_coverage}
    \end{subfigure}

    \begin{subfigure}{0.45\linewidth}
        \centering
        \includegraphics[width=\linewidth]{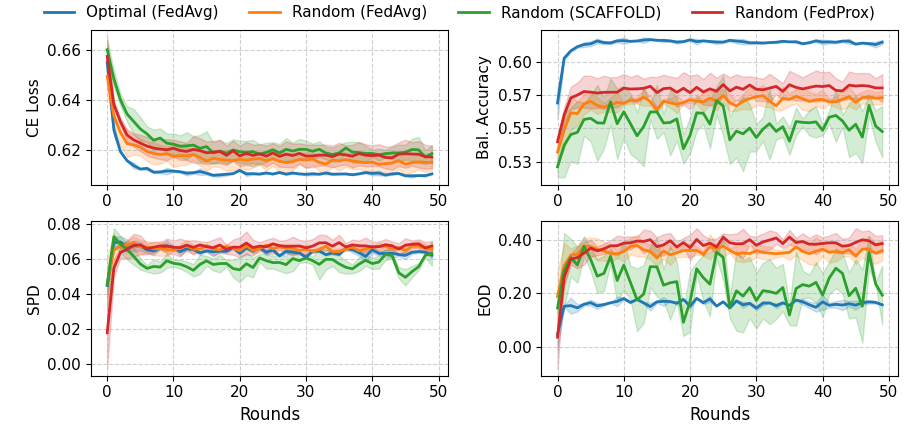}
        \caption{TravelTime}
        \label{fig:exp2_travel_time}
    \end{subfigure}
    \hspace{0.5cm}
    \begin{subfigure}{0.45\linewidth}
        \centering
        \includegraphics[width=\linewidth]{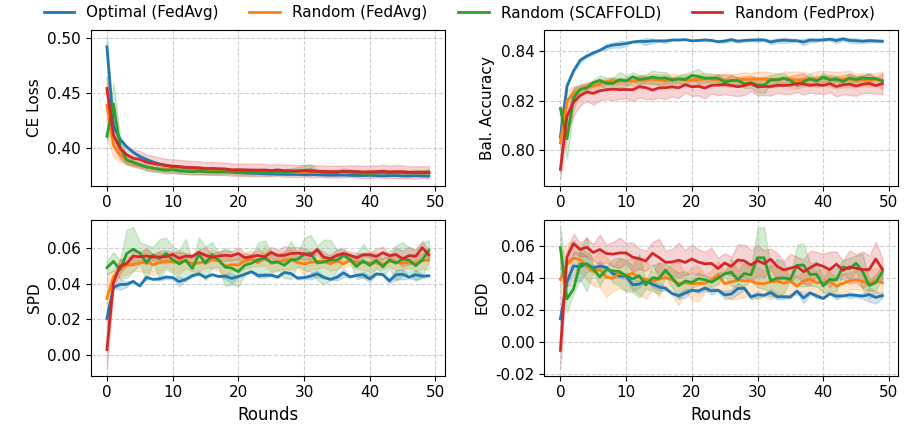}
        \caption{ACSEmployment}
        \label{fig:exp2_employment}
    \end{subfigure}

    \caption{Convergence and fairness of our optimal federations compared with random ones, with FedAvg, FedProx, and SCAFFOLD, for ACSIncome, ACSEmployment, ACSPublicCoverage, and ACSTravelTime with $k=15$.}
    \label{fig:exp_1_k_15}
\end{figure*}

\begin{figure*}[h]
    \centering
    \begin{subfigure}{0.45\linewidth}
        \centering
        \includegraphics[width=\linewidth]{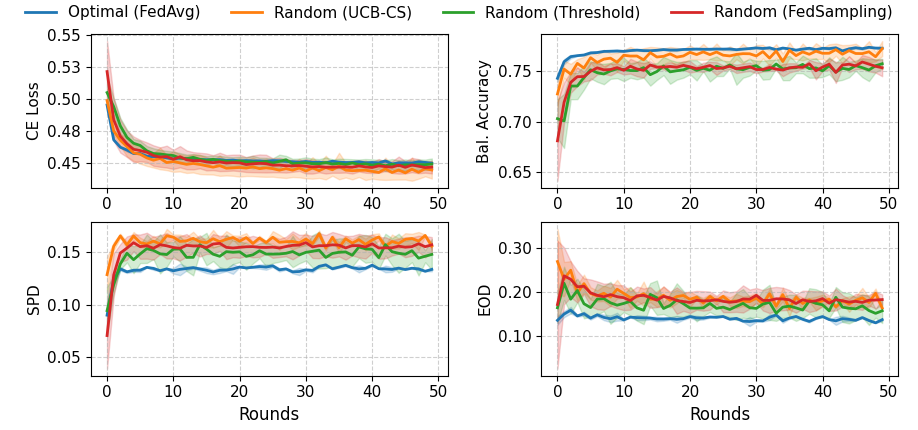}
        \caption{ACSIncome}
        \label{fig:exp2_income}
    \end{subfigure}
    \hspace{0.5cm}
    \begin{subfigure}{0.45\linewidth}
        \centering
        \includegraphics[width=\linewidth]{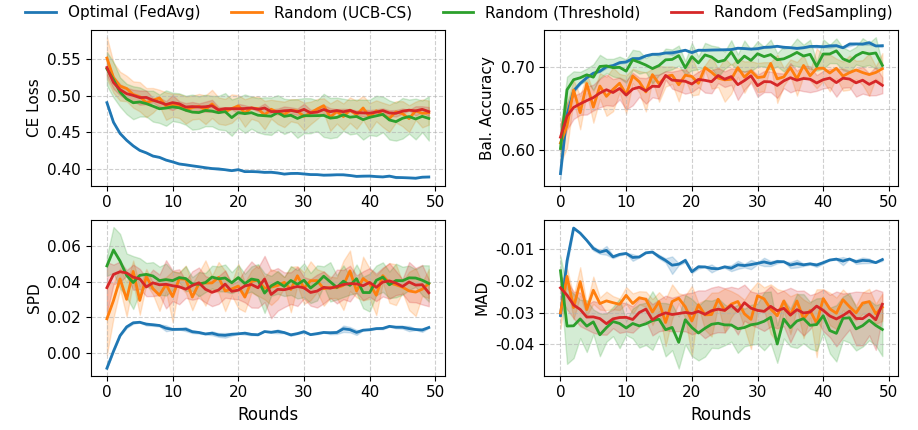}
        \caption{PublicCoverage}
        \label{fig:exp2_public_coverage}
    \end{subfigure}

    \begin{subfigure}{0.45\linewidth}
        \centering
        \includegraphics[width=\linewidth]{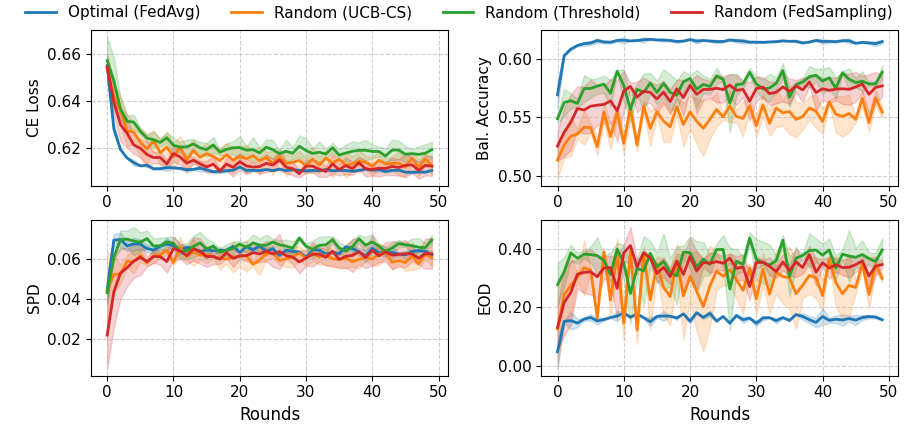}
        \caption{TravelTime}
        \label{fig:exp2_travel_time}
    \end{subfigure}
    \hspace{0.5cm}
    \begin{subfigure}{0.45\linewidth}
        \centering
        \includegraphics[width=\linewidth]{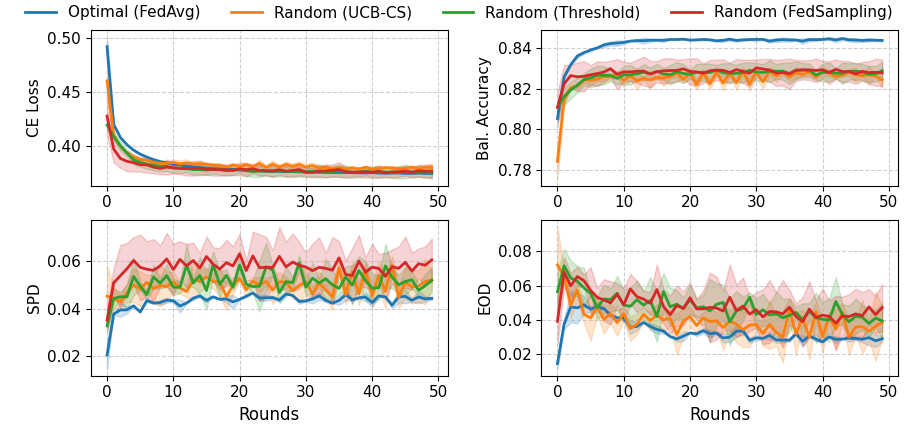}
        \caption{ACSEmployment}
        \label{fig:exp2_employment}
    \end{subfigure}

    \caption{Convergence and fairness of our optimal federations compared with random ones, with sampling strategies: UCB-CS \cite{cho2020bandit}, FedSampling \cite{qi2023fedsampling}, and Threshold-based participation \cite{ribero2020communication}, for ACSIncome, ACSEmployment, ACSPublicCoverage, and ACSTravelTime with $k=15$.}
    \label{fig:exp_2_k_15}
\end{figure*}